\documentclass{pkumelon}

\usepackage[utf8]{inputenc}
\usepackage{url}
\usepackage{amsfonts}
\usepackage{nicefrac}
\usepackage{xspace}
\usepackage{amsmath}
\usepackage{amsthm}
\usepackage{amssymb}
\usepackage{multicol}
\usepackage{makecell}
\usepackage{enumitem}
\usepackage{wrapfig}
\usepackage{dsfont}
\usepackage{epstopdf}
\usepackage{colortbl}
\usepackage{algorithm}
\usepackage{threeparttable}
\usepackage{mathtools}
\usepackage{float}
\usepackage{mathrsfs}
\usepackage{algpseudocode}
\usepackage{tabularray}
\usepackage{quiver}
\usepackage{diagbox}
\usepackage{needspace}

\renewcommand{\titlefont}{\fontsize{17}{20}\selectfont}
\renewcommand\affiliationformat[2][]{{\affiliationfont {\sffamily $^{#1}$#2}}}


\makeatletter
\renewcommand\authorformat[2][]{%
  \authorfont {\sffamily
    \mbox{\ifstrempty{#1}{#2}{#2$^{#1}$}}%
  }%
}
\makeatother

\newtcolorbox{limitationbox}{
    enhanced,
    breakable,
    colback=boxbg,
    frame hidden,
    boxrule=0pt,
    borderline west={1.2pt}{0pt}{reaslab},
    sharp corners,
    left=6pt,
    right=2pt,
    top=3pt,
    bottom=3pt,
    before skip=5pt,
    after skip=5pt
}
\newenvironment{limbox}{\begin{limitationbox}}{\end{limitationbox}}

\restylefloat{algorithm} 
\theoremstyle{plain}
\newtheorem{assumption}{Assumption}

\newtheorem{theorem}{Theorem}
\newtheorem{remark}{Remark}
\newtheorem{lemma}{Lemma}

\newtheorem{corollary}{Corollary}

\newcommand{\Nest}{\mathcal{N}}

\newcommand{\GNMR}{\text{GNMR}}

\newcommand{\dGNMR}{\Delta\text{-GNMR}}

\usepackage{tabularray}
\usepackage{xcolor,pifont}
\usepackage{MnSymbol,wasysym}

\newcommand{\xmark}{\textcolor{red}{\large\ding{55}}}
\newcommand{\xmarkk}{\textcolor{red}{\large\ding{55} }}

\title{GNMR: Runtime Stability Control for Low-Precision Large Language Model Training}

\author[1,*]{Boao Kong}
\author[1,*]{Weichen Jia}
\author[1,*]{Engao Zhang}
\author[1,*]{Guohong Li}
\author[2]{Yonghan Dong}
\author[2]{Yao Wang}
\author[2]{Yaoyuan Wang}
\author[2]{Yunke Peng}
\author[1,\P]{Kun Yuan}

\affiliation[1]{Peking University}
\affiliation[2]{Huawei Technologies Ltd.}
\contribution[*]{Equal Contribution}
\contribution[\P]{Corresponding author}
\checkdata[Emails]{\email{kongboao@stu.pku.edu.cn} (Boao Kong), \email{kunyuan@pku.edu.cn} (Kun Yuan)}

\abstract{
Training stability is a key bottleneck in low-precision language model training: efficient low-cost paths can still produce short-lived numerical risks at a small set of operators. We formulate this as runtime stability control and present \textbf{\underline{G}}radient \textbf{\underline{N}}orm-to-\textbf{\underline{M}}ean \textbf{\underline{R}}atio (\textbf{GNMR}), a lightweight controller that compares each recoverable unit's current gradient norm with its historical mean. Together with $\Delta$-GNMR for abrupt short-window increases, GNMR maps local risk signals to bounded recovery actions under a hard $\mathrm{maxO}$ budget and a short lock interval, without changing the numerical format, kernel, or backend recipe. Across activation-quantization stress, DeepSeek-style recipe-level training, and LLaMA-2 13B fine-tuning, GNMR preserves high-fidelity quality with sparse, budgeted recovery. These results support GNMR as a backend-agnostic controller to improve low-precision training stability while preserving low-cost execution.
}

\begin{document}

\maketitle

\section{Introduction}
Training stability is a central constraint in training and adaptation of large language model (LLM). To reduce computation, memory, and communication costs, modern training stacks use low-cost execution paths. These paths may appear as mixed-precision training \citep{micikevicius2017mixed}, FP8 training recipes \citep{peng2023fp8}, compression of activation or optimization state \citep{xi2024coat}, or quantized adaptation \citep{dettmers2023qlora}. Previous work has shown that such paths can preserve acceptable model quality when combined with suitable formats, scaling rules, outlier handling, and high-fidelity safeguards. The question we study is not whether low-cost execution is useful on average. Instead, we ask whether training can detect and recover from local stability risks when low-precision perturbations become harmful at runtime.

From a stability angle, low-precision perturbations do not affect all layers, operators, and training steps equally. Most low-cost execution may be safe for most steps, while a few local units can become risky in short windows. LLM.int8~\citep{dettmers20218} identifies outlier features as a key difficulty in quantization of LLMs, and SmoothQuant~\citep{xiao2023smoothquant} treats activation outliers as a central obstacle to accurate low-bit execution. Reduced-precision training can be sensitive to stability and hyperparameter choices~\citep{lee2024fp8}. More broadly, loss and gradient spikes are known failure modes in large-scale pre-training \citep{takase2023spike} and have motivated optimizer-level stabilization methods \citep{huang2025spam}.
\begin{figure*}[t!]
\centering
\includegraphics[width=0.96\textwidth]{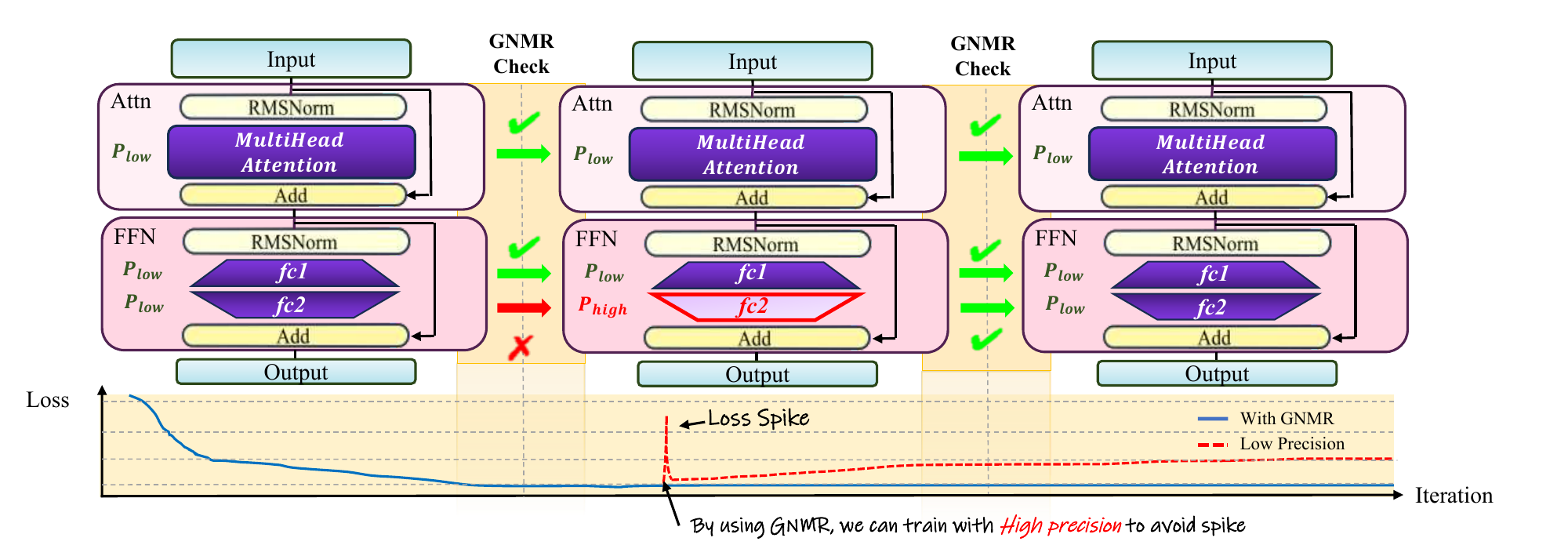}
\caption{Overview of GNMR as a runtime stability controller. GNMR monitors local risk signals and routes selected recoverable units from the low-cost path to the higher-fidelity recovery path under a fixed active budget.}
\label{fig:pipeline}
\end{figure*}

These observations make stability under low precision a runtime control problem, not only a static configuration problem. A layer may be stable for many updates, become risky for a few steps when gradients, activation statistics, scaling factors, or optimizer states shift, and then return to a stable regime. In the same step, most of the other layers may remain safe. A useful controller should therefore identify which local units are becoming unstable, decide whether recovery is needed now, and apply recovery under a fixed overhead budget.

Existing work addresses low-precision numerical issues in several directions. Low-bit training recipes extend practical execution to FP4-style regimes \citep{wang2025optimizing} and NVFP4-style regimes \citep{abecassis2025pretraining}. Training stability methods include gradient clipping \citep{pascanu2013difficulty}, gradient-noise-scale analysis \citep{mccandlish2018empirical}, and GSNR-style signals \citep{jiang2023accelerating}. Dynamic precision methods study adaptive bit-width choices over time or structure \citep{yu2022ldplearnabledynamicprecision}. These methods are valuable, but they mainly focus on the format, recipe, optimizer, or offline calibration level. They do not directly answer the controller-level question studied here: given an existing low-cost path and a more accurate recovery path, how should training decide online which local units to recover under a fixed budget?

\subsection{Challenges in Runtime Stability Control under Low Precision}

The stability-control view above leads to four requirements that existing low-precision approaches do not jointly satisfy.

\noindent \textbf{(L1) Sparse and non-stationary runtime risk.}
Low-precision instability is rarely uniform across all layers and steps. A unit may be safe for most of training process but risky during a short window. Static or semi-static strategies may miss such transient risks, or protect too much of the model with conservative safeguards.

\noindent \textbf{(L2) Local and lightweight risk signals.}
Existing stability indicators include loss spikes \citep{takase2023spike}, global gradient norms \citep{pascanu2013difficulty}, gradient-noise statistics \citep{mccandlish2018empirical}, GSNR-style signals \citep{jiang2023accelerating}, and outlier statistics \citep{xiao2023smoothquant}. However, a runtime controller needs more than a general warning signal. It needs a signal that is cheap to compute, available in every single step, local enough to identify the risky unit, and comparable across modules.

\noindent \textbf{(L3) Budgeted and stable intervention.}
Recovery actions are not free. Sending more units through a higher-fidelity path can increase peak memory, casting overhead, and routing complexity. Frequent switching can also cause thrashing. A controller must therefore decide which units are risky, how many can be recovered at each step, and how long recovery should persist.

\noindent \textbf{(L4) Separation between control quality and backend implementation.}
Low-precision studies often evaluate complete systems where formats, kernels, scaling rules, optimizers, and fallback policies change together. Such evaluations are useful, but they can hide the controller question. To evaluate runtime stability control, we need to separate trigger quality, budget behavior, recipe-level integration, and large-model boundedness from claims about new kernels or universal hardware speedups.

\subsection{Contributions}

To address these challenges, we present \textbf{{GNMR}}, short for \textbf{\underline{G}}radient \textbf{\underline{N}}orm-to-\textbf{\underline{M}}ean \textbf{\underline{R}}atio, a runtime stability controller for low-precision language model training, as illustrated in Figure~\ref{fig:pipeline}. GNMR monitors operator-wise gradient deviations and applies temporary high-fidelity recovery only to risky low-precision-eligible units under a fixed budget. Its core contribution is a stability-aware risk monitor and a budgeted closed-loop control policy operating over existing low-cost and recovery paths. Our contributions can be summarized as:

\noindent \textbf{(C1) Runtime stability-control formulation under low precision.}
We formulate local instability in low-precision language model training as a constrained online control problem. Given a low-cost execution path and a higher-fidelity recovery path, the controller observes runtime statistics of eligible units and selects a sparse subset for temporary recovery under a memory or compute budget. This formulation unifies risk identification, recovery decisions, and budget constraints while decoupling the control policy from backend implementation. \textbf{This addresses (L1) and (L4).}

\noindent \textbf{(C2) Dual-timescale GNMR risk monitor.}
We design GNMR as a lightweight operator-wise statistic that compares a unit's current gradient norm with its historical mean, capturing long-horizon deviations from typical training behavior. We further introduce $\Delta$-GNMR, which normalizes GNMR by a short recent window to detect abrupt local changes. Together, they provide a real-time, dimensionless, and local risk signal for runtime control. \textbf{This addresses (L2).}

\noindent \textbf{(C3) Budgeted recovery policy with analysis.}
We design a policy that recovers only top-risk units, uses $\mathrm{maxO}$ to control instantaneous overhead, and applies a short $T_{\mathrm{lock}}$ interval to reduce thrashing. We also provide a one-step expected-descent analysis showing why selective recovery can reduce low-precision perturbation on critical updates. \textbf{This addresses (L3).}

\noindent \textbf{(C4) Controller-level evidence across low-precision settings.}
We evaluate GNMR in a controlled activation-quantization stress bench, a TE-backed low-precision training setting \citep{perez2023traininginferencelargelanguage}, and a large-model post-training task. The experiments examine trigger quality, budget behavior, quality preservation, and engineering overhead, supporting GNMR as a runtime stability controller that separates trigger quality, budget behavior, recipe-level integration, and downstream preservation from backend implementation choices. \textbf{This addresses (L4).}

A comprehensive review of \textbf{Related Work} is provided in Appendix~\ref{sec: appendix_related_works}.

\section{Preliminary: Operator-wise Runtime Monitoring}
\label{sec:preliminaries}

We use an operator-level view of transformer training to define the local runtime statistics used by our controller. 
This view separates the compositional structure of the network from within-layer parallelism, but we keep only the notation needed in the main text. 
Detailed instantiations for attention, feed-forward networks, normalization, residual paths, aggregation, and backpropagation are deferred to Appendix~\ref{appendix: operator-level modeling}.

\noindent \textbf{Training objective.}
Let $F(\cdot;\Theta)$ be the model mapping and $r(\cdot;\theta^{(r)})$ the readout to logits. 
For data $(x,y)\!\sim\!\mathcal D$, we write the training objective as
\vspace{-1mm}
\begin{align*}
\label{eq: modeling_1}
L(\Theta,\theta^{(r)})
=
\mathbb{E}_{(x,y)\sim\mathcal D}
\Big[
\ell\big(r(F(x;\Theta);\theta^{(r)}),y\big)
\Big],
\end{align*}

\vspace{-1mm}
\noindent where $\Theta$ contains the parameters of the model's operator blocks.

\noindent \textbf{Operator-level decomposition.}
We write the network as a composition of $U$ nested operators:
\begin{equation}
\label{eq: modeling_2}
\begin{aligned}
    y^{[0]}=x,\quad y^{[u]}=\Nest^{[u]}\!\left(y^{[u-1]};\Xi^{[u]}\right),
\quad u=1,\ldots,U .
\end{aligned}
\end{equation}
Each block $\Nest^{[u]}$ may contain multiple sibling operators whose outputs are combined by an aggregation operation. 
This covers multi-head attention, feed-forward modules, normalization, and residual paths. 
We leave the explicit sibling-operator and aggregation formulas to Appendix~\ref{appendix: operator-level modeling}.

\noindent \textbf{Monitored runtime units.}
The decomposition above induces a set of monitored units $\mathcal{B}$. 
A unit $b\in\mathcal{B}$ may correspond to an operator, branch, block, or layer, depending on the implementation granularity. 
We denote its parameters by $\theta_b$. 
At training step $t$, let $\widehat L_t$ be the stochastic mini-batch loss. 
The local gradient and its norm are
\begin{align}
\label{eq:local_gradient}
g_{b,t}
=
\nabla_{\theta_b}\widehat L_t,
\qquad
n_{b,t}
=
\|g_{b,t}\|_2 .
\end{align}

\vspace{-1mm}
\noindent The scalar $n_{b,t}$ is the basic runtime statistic used in Section~\ref{sec:gnmr_controller}. 
Because raw gradient norms are not directly comparable across operators with different typical scales, the next section introduces a history-normalized statistic for local stability monitoring.

\section{GNMR Runtime Stability Controller}
\label{sec:gnmr_controller}

Section~\ref{sec:preliminaries} defines the operator-wise gradient norm
$n_{b,t}$ as a local statistic available during training.
Although $n_{b,t}$ is local and cheap to obtain, its raw value cannot be used as a risk score directly, since different operators have different typical scales.
A useful risk criterion should therefore measure whether a unit's current update is unusual relative to its own history.

We instantiate this criterion with two history-normalized risk scores: GNMR for long-horizon deviations and $\Delta$-GNMR for short-window changes.
These scores are then used to drive the budgeted recovery policy in Section~\ref{sec:budgeted_recovery}.

\subsection{Long-horizon Risk Score: GNMR}
\label{sec:gnmr}

\noindent GNMR, short for Gradient Norm-to-Mean Ratio, measures how unusual the current gradient norm of a monitored unit is relative to its own historical scale.
For an operator parameter $\theta^{[u,k]}$, this corresponds to comparing its current gradient norm with its running historical mean:

\begin{align}
\GNMR_t(\theta^{[u,k]}) := \dfrac{\left\Vert g_t^{[u,k]} \right\Vert}{\sum_{\tau=1}^{t-1}\left\Vert g_{\tau}^{[u,k]} \right\Vert/(t-1)},
\end{align}

\noindent where $g_t^{[u,k]}$ denotes the evaluated gradient at the $t$-th step with respect to $\theta^{[u,k]}$. 
We set $\GNMR_1(\theta^{[u,k]})=1$.
Instead of comparing raw gradient norms across operators, GNMR compares each unit with its own past behavior, making the score dimensionless and more suitable for ranking local runtime risks across heterogeneous modules.

Compared with stability indicators such as gradient spike score (GSS) \cite{huang2025spam}, gradient signal-to-noise ratio (GSNR) \cite{jiang2023accelerating}, and Jacobian spectral norm \cite{takase2023spike}, GNMR is designed for online operator-wise monitoring.
GSS relies on full-process gradient statistics, GSNR requires large-batch gradients, and Jacobian-based quantities introduce substantial overhead.
In contrast, GNMR only uses the current gradient norm and a running historical average, as summarized in Algorithm~\ref{alg:get_gnmr}.

\begin{algorithm}[t]
  \caption{Online computation of GNMR and $\Delta$-GNMR for one monitored unit}
  \label{alg:get_gnmr}
  \begin{algorithmic}
  \Require Gradient $g_{b,t}$, historical mean $A_{b,t-1}$, recent GNMR queue $Q_b$, window length $\delta$, step $t$, and small constant $\epsilon$.
  \State $n_{b,t} \gets \|g_{b,t}\|_2$.
  \If{$t=1$ \textbf{or} $A_{b,t-1}=0$}
      \State $\GNMR_{b,t} \gets 1$.
  \Else
      \State $\GNMR_{b,t} \gets n_{b,t}/(A_{b,t-1}+\epsilon)$.
  \EndIf
  \If{$|Q_b|=\delta$}
      \State $\Delta\text{-}\GNMR_{b,t} \gets \GNMR_{b,t}/(\mathrm{mean}(Q_b)+\epsilon)$.
  \Else
      \State $\Delta\text{-}\GNMR_{b,t} \gets 1$.
  \EndIf
  \State $A_{b,t} \gets A_{b,t-1} + (n_{b,t}-A_{b,t-1})/t$.
  \State Push $\GNMR_{b,t}$ into $Q_b$ and keep only the most recent $\delta$ values.
  \State \Return $\GNMR_{b,t}$, $\Delta\text{-}\GNMR_{b,t}$, $A_{b,t}$, and $Q_b$.
  \end{algorithmic}
\end{algorithm}

\begin{algorithm}[t]
  \caption{Budgeted runtime recovery from GNMR risk scores}
  \label{alg:budgeted_recovery}
  \begin{algorithmic}
  \Require Recoverable units $\mathcal B_{\mathrm{rec}}$, risk thresholds $\alpha_t,\beta_t$, active budget $\mathrm{maxO}$, lock length $T_{\mathrm{lock}}$, and lock counters $\mathrm{Lock}_b$ for $b\in\mathcal B_{\mathrm{rec}}$.
  \For{each training step $t$}
      \For{each recoverable unit $b\in\mathcal B_{\mathrm{rec}}$}
          \State Compute $\GNMR_{b,t}$ and $\Delta\text{-}\GNMR_{b,t}$ by Algorithm~\ref{alg:get_gnmr}.
          \State $R_{b,t}\gets \max\{\GNMR_{b,t}/\alpha_t,\Delta\text{-}\GNMR_{b,t}/\beta_t\}$.
      \EndFor
      \State $\mathcal C_t\gets \{b:R_{b,t}>1\}\cup\{b:\mathrm{Lock}_b>0\}$.
      \State $\mathcal A_t\gets \operatorname{Top}_{\mathrm{maxO}}(\mathcal C_t; R_{b,t})$.
      \For{each recoverable unit $b\in\mathcal B_{\mathrm{rec}}$}
          \If{$b\in\mathcal A_t$}
              \State Use the recovery path for $b$.
              \If{$R_{b,t}>1$}
                  \State $\mathrm{Lock}_b\gets T_{\mathrm{lock}}$.
              \Else
                  \State $\mathrm{Lock}_b\gets \max(\mathrm{Lock}_b-1,0)$.
              \EndIf
          \Else
              \State Use the low-cost path for $b$ and set $\mathrm{Lock}_b\gets 0$.
          \EndIf
      \EndFor
  \EndFor
  \end{algorithmic}
\end{algorithm}

\subsection{Short-window Risk Score: $\Delta$-GNMR}
\label{sec:delta_gnmr}

\noindent GNMR uses long-horizon history, which is useful for measuring persistent deviation but may smooth out abrupt local changes.
To capture short-window instability, we introduce $\Delta$-GNMR, which compares the current GNMR value with its recent-window average:

\begin{align}
\Delta\text{-GNMR}_t(\theta^{[u,k]}) := \dfrac{\text{GNMR}_t(\theta^{[u,k]})}{\sum_{\tau=1}^{\delta}\text{GNMR}_{t-\tau}(\theta^{[u,k]})/\delta},
\end{align}

\noindent where $\delta$ is the short-window length.
Thus, GNMR asks whether a unit deviates from its long-term behavior, while $\Delta$-GNMR asks whether the current risk rises sharply relative to recent steps.

Similar to GNMR, $\Delta$-GNMR metric can be computed online with a small queue of recent GNMR values and does not require extra backward passes or full-process statistics.
Algorithm~\ref{alg:get_gnmr} gives the concrete computation for a generic monitored unit $b$; for an operator parameter $\theta^{[u,k]}$, one can take $b=(u,k)$.
The scores themselves do not define a numerical format or kernel; they provide local risk estimates that are later converted into recovery actions by the controller.

\subsection{From Risk Scores to Budgeted Recovery}
\label{sec:budgeted_recovery}
GNMR and \(\Delta\)-GNMR provide local risk estimates, but the controller must turn these scores into bounded recovery actions. We distinguish monitored units from recoverable units. A monitored unit has runtime statistics, while a recoverable unit is one for which the backend training stack provides both a low-cost path and a higher-fidelity recovery path. Units fixed by the backend recipe remain outside the action space.

At step \(t\), for each recoverable unit \(b\), we combine the two risk scores as
\[
R_{b,t}
=
\max\left\{
\frac{\GNMR_{b,t}}{\alpha_t},
\frac{\Delta\text{-}\GNMR_{b,t}}{\beta_t}
\right\}.
\]
Here, \(\alpha_t\) and \(\beta_t\) control the risk sensitivity of long-horizon deviations and short-window changes. Candidate units are those with \(R_{b,t}>1\), together with units still prioritized by the lock mechanism. The controller then selects the final active recovery set with the top $\mathrm{maxO}$ operators as
\[
\mathcal A_t
=
\operatorname{Top}_{\mathrm{maxO}}(\mathcal C_t; R_{b,t}),
\]
where \(\mathcal C_t\) is the candidate pool. Only units in \(\mathcal A_t\) use the recovery path; all remaining recoverable units use the low-cost path. Thus, \(\mathrm{maxO}\) caps the number of units active in recovery at each step, matching the budgeted set in Theorem~\ref{thm:op-descent-gain-main}.

The lock mechanism stabilizes active-set membership but does not bypass this active budget. When a unit is selected, its lock counter is reset to \(T_{\mathrm{lock}}\); locked units remain in the candidate pool for a short interval, but the final active set is still chosen under the same top-\(\mathrm{maxO}\) cap. This prevents frequent switching while preserving a fixed instantaneous recovery budget. 
Algorithm~\ref{alg:budgeted_recovery} summarizes the resulting closed-loop recovery procedure.

\subsection{Controller Interface, Budget, and Calibration}
\label{sec:controller_budget_calibration}

\noindent\textbf{Controller interface and backend constraints.}
GNMR operates on backend-provided low-cost and higher-fidelity recovery paths for a subset of recoverable units. 
It assumes that the backend provides a low-cost path and a higher-fidelity recovery path for a subset of recoverable units, which can be operators, blocks, or layers depending on the implementation. 
Units fixed by the backend recipe may still be monitored, but they are not adjusted by the controller. 
When operator-level recovery may create inconsistent execution within a block, we treat the whole transformer block as the recoverable unit, so all recoverable low-cost operations inside the block share the same recovery decision. 
This interface lets the same risk monitor work across controlled activation-quantization stress tests, TE-backed low-precision training, and large-model post-training stress tests, while keeping backend-specific constraints explicit.

\noindent\textbf{Active budget and calibration.}
Recovery incurs extra memory compared with the low-cost path, especially when multiple units become risky in the same step. 
The budget $\mathrm{maxO}$ is a hard cap on the active recovery set: at each
step, the controller executes at most $\mathrm{maxO}$ recoverable units on the
recovery path. Let $c_b$ denote the additional saved-activation cost of
recovering unit $b$. Then
\vspace{-1mm}
\begin{equation}
\label{eq:generic-memory-bound}
\begin{aligned}
\Delta M_t
&=
\sum_{b\in\mathcal A_t} c_b
\le
\mathrm{maxO}\cdot c_{\max},\quad
c_{\max}
=
\max_{b\in\mathcal B_{\mathrm{rec}}} c_b .
\end{aligned}
\end{equation}
Appendix~\ref{app:memory} instantiates $c_b$ using saved-activation bitwidths and activation footprints. The lock interval $T_{\mathrm{lock}}$ reduces switching by keeping recent risky units eligible for short-term retention, but the final active set $\mathcal A_t$ is always selected under the same $\mathrm{maxO}$ cap.

The lock interval \(T_{\mathrm{lock}}\) stabilizes active-set membership over time: recently selected units remain in the candidate pool for a short window and can be retained if they remain among the top-risk units, but the final active set is always selected under the same \(\mathrm{maxO}\) cap. 
Thus, the lock mechanism reduces unnecessary switching without bypassing the recovery budget.

The thresholds \(\alpha_t\) and \(\beta_t\) control risk sensitivity. 
The former governs long-horizon GNMR deviations, while the latter governs short-window \(\Delta\)-GNMR changes. 
In practice, a higher warm-up threshold reduces noisy early triggers when the history is short, while scheduled relaxation or warm-up calibration maintains sensitivity after running statistics stabilize. 
Section~\ref{section: theoretical} provides a bound-level motivation connecting active recovery, perturbation reduction, and threshold calibration.

\section{Theoretical Motivation}
\label{section: theoretical}

We give a bound-level motivation for budgeted recovery. Low-cost
execution may perturb forward computation, saved activations, or
backward signals, and such perturbations can propagate through the
computation graph. We therefore do not model low-precision error as
independent zero-mean gradient noise. Instead, for each recoverable unit
\(b\) and path \(q\in\{\mathrm{low},\mathrm{rec}\}\), let
\(\mathcal E_q^b(\rho)\) denote an operator-wise perturbation penalty
that upper-bounds its contribution to the one-step smoothness bound at
risk level \(\rho\). The standing assumptions, including \(L\)-smoothness,
standard SGD noise, and the perturbation envelope, are given in
Appendix~\ref{appendix: assumptions_and_proofs}.

Define the recovery gap
\[
\Delta_s^b
=
\Delta\mathcal E^b(\rho_s^b)
=
\mathcal E_{\mathrm{low}}^b(\rho_s^b)
-
\mathcal E_{\mathrm{rec}}^b(\rho_s^b).
\]
We assume \(\Delta_s^b\ge 0\) on the risky regimes targeted by the
controller.

\begin{table*}[t!]
\caption{Validation perplexity ($\downarrow$) in the controlled activation-quantization stress bench for LLaMA-2 pre-training. The low-cost path quantizes saved activations of recoverable projection operators; GNMR+\(\Delta\)-GNMR selectively routes risky units to the recovery path. \xmarkk means the model fails to converge.}
\vspace{-1.5mm}
\label{table:activation quantilization}
\centering
\setlength{\tabcolsep}{10pt}
\renewcommand{\arraystretch}{1.1}
\label{table:activation quantilization}
\begin{threeparttable}
{\small
\begin{tabular}{ccccccc} 
\toprule
\multirow{2}{*}{\textbf{precision}} & \multirow{2}{*}{$\alpha_t$} & \multirow{2}{*}{$\beta_t$} & 60M & 130M & 350M & 1.3B \\
\cline{4-7}
&&& 1.1B tokens & 2.2B tokens& 6.4B tokens & 13.1B tokens \\ \midrule
\textbf{4-bit} & N.A. & N.A. & 67.22 & 169.21& 26.56 & \xmark \\\midrule
\multirow{3}{*}{\textbf{4-bit/8-bit}} & \cellcolor{cyan!30}1.5 & \cellcolor{cyan!30}1.3 & \cellcolor{cyan!30}\textbf{30.59} & \cellcolor{cyan!30}\textbf{24.66} & \cellcolor{cyan!30}18.84 & \cellcolor{cyan!30}15.71 \\
 & \cellcolor{orange!30}2.0&\cellcolor{orange!30}1.5 & \cellcolor{orange!30}30.66 & \cellcolor{orange!30}24.86 & \cellcolor{orange!30}19.07 & \cellcolor{orange!30}N.A. \\
 & \cellcolor{green!30}3.0&\cellcolor{green!30}2.0 & \cellcolor{green!30}30.68 & \cellcolor{green!30}25.33 & \cellcolor{green!30}18.77 & \cellcolor{green!30}N.A.  \\ \midrule
\textbf{8-bit} & N.A. & N.A. & 30.88 & 25.13 & \textbf{18.59} & 15.68 \\\midrule
\textbf{16-bit} & N.A. & N.A. & 30.79 & 25.06 & 18.80 & \textbf{15.56} \\
\bottomrule
\end{tabular}}
\end{threeparttable}
\end{table*}

\Needspace{9\baselineskip}
\begin{limbox}
\begin{theorem}[Bound-level gain of budgeted recovery]
\label{thm:op-descent-gain-main}
Let \(\mathcal A_s\subseteq\mathcal B_{\mathrm{rec}}\) be the active
recovery set at step \(s\), with \(|\mathcal A_s|\le\mathrm{maxO}\).
Let \(\bar{\mathcal U}_s^{\mathrm{Low}}\) be the one-step upper bound
under the always-low-cost path, and
\(\bar{\mathcal U}_s^{\mathrm{Rec}}(\mathcal A_s)\) the bound when
units in \(\mathcal A_s\) use recovery. If we denote $\Delta^{\mathcal U}_s
:=
\bar{\mathcal U}_s^{\mathrm{Low}}
-
\bar{\mathcal U}_s^{\mathrm{Rec}}(\mathcal A_s),$ then under the standing assumptions,
\[
\Delta^{\mathcal U}_s
=
\sum_{b\in\mathcal B_{\mathrm{rec}}}
\Delta\mathcal E^b(\rho_s^b)
\mathbf 1\{b\in\mathcal A_s\}.
\]
Moreover, for the transition from step \(t\) to \(t+1\), suppose the
newly selected set satisfies \(\mathcal S_t\subseteq\mathcal A_{t+1}\),
\begin{align*}
    \Pr(\rho_{t+1}^b\ge\tau_b\mid b\in\mathcal S_t)\ge\pi_b,\quad\Delta\mathcal E^b(\rho)\ge\Delta\mathcal E^b(\tau_b)\mathbf 1\{\rho\ge\tau_b\}.
\end{align*}
Then it holds that
\begin{equation}
\begin{aligned}
\mathbb E[\Delta^{\mathcal U}_{t+1}]
\ge
\sum_{b\in\mathcal B_{\mathrm{rec}}}
\Pr(b\in\mathcal S_t)\pi_b\Delta\mathcal E^b(\tau_b).
\end{aligned}
\end{equation}
(See the proof in Appendix~\ref{sec:the_part2}.)
\end{theorem}
\end{limbox}

\Needspace{9\baselineskip}
\begin{limbox}
\begin{corollary}[Oracle target under an active budget]
\label{cor:oracle_budget}
If the latent gaps \(\Delta_s^b\) were observable, the oracle recovery set under budget
\(|\mathcal A_s|\le\mathrm{maxO}\) would solve
\[
\mathcal A_s^\star
\in
\arg\max_{\mathcal A\subseteq\mathcal B_{\mathrm{rec}},\,|\mathcal A|\le\mathrm{maxO}}
\sum_{b\in\mathcal A}\Delta_s^b.
\]
Thus, it selects the largest positive gaps up to the total budget.
\end{corollary}
\end{limbox}

\noindent Corollary~\ref{cor:oracle_budget} identifies the latent target of a budgeted controller: ranking units by their unobserved recovery gaps under the active budget. 
Since \(\Delta_s^b\) is not observable online, GNMR and \(\Delta\)-GNMR instantiate this oracle view with lightweight runtime scores for ranking risky units. 
The lock interval \(T_{\mathrm{lock}}\) further stabilizes active-set membership, yielding a budgeted online policy consistent with the descent-gain mechanism in Theorem~\ref{thm:op-descent-gain-main}.

\Needspace{9\baselineskip}
\begin{limbox}
\begin{theorem}[Historical-mean concentration under fixed GNMR thresholds]
\label{thm:GNMR-fixed}
For unit \(b\) and \(t\ge2\), let \(n_{b,t}=\|g_{b,t}\|_2\). For any fixed \(\alpha=1+\varepsilon>1\), under the sub-exponential tail
condition in Appendix~\ref{appendix: assumptions_and_proofs}, there is
\(c>0\) such that

\begin{equation}
\begin{aligned}
\Pr(\GNMR_{b,t}&\ge \alpha)\le\exp\left(-c\min\left\{\frac{\varepsilon^2}{\kappa_b^2},\frac{\varepsilon}{\kappa_b}\right\}\right)+2\exp\left(-c(t-1)\min\left\{\frac{(\varepsilon/\alpha)^2}{\kappa_b^2},\frac{\varepsilon/\alpha}{\kappa_b}\right\}\right).
\end{aligned}
\end{equation}
(See the proof in Appendix~\ref{sec:the_part3}.)
\end{theorem}
\end{limbox}

\begin{table*}[t!]
\caption{The validation perplexity ($\downarrow$) for the pre-training LLaMA-2 models for variance model size with DeepSeek-style mixed-precision strategy under different precision strategy. $\alpha_t$ and $\beta_t$ represent the threshold of GNMR and $\Delta$-GNMR, respectively.}
\vspace{-1.5mm}
\label{table:llama2_activation_quant}
\centering
\setlength{\tabcolsep}{6pt}
\renewcommand{\arraystretch}{1.1}
\begin{threeparttable}
{\small
\begin{tabular}{ccc|ccccccccc} 
\toprule
\multirow{2}{*}{\textbf{precision}} & \multirow{2}{*}{$\alpha_t$} & \multirow{2}{*}{$\beta_t$} & \multicolumn{8}{c}{Training tokens (B)} \\
\cline{4-12}
&&& 1.3 & 2.6 & 3.9 & 5.2 & 6.6 & 7.9 & 9.2 & 10.5 & 13.1 \\ \midrule
\cellcolor{cyan!30}\textbf{6-bit/8-bit}& \cellcolor{cyan!30}1.5/1.1$^\Diamond$&\cellcolor{cyan!30}1.5  & \cellcolor{cyan!30}28.52 &\cellcolor{cyan!30}21.17 &\cellcolor{cyan!30}19.07 &\cellcolor{cyan!30}17.88 &\cellcolor{cyan!30}17.06 &\cellcolor{cyan!30}16.43 &\cellcolor{cyan!30}15.93 &\cellcolor{cyan!30}15.57 &\cellcolor{cyan!30}15.20 \\ \midrule
\textbf{8-bit} & N.A. & N.A. & 29.54 &21.57 &19.25 &17.96 &17.11 &16.45 &15.94 &15.56 &15.18 \\\midrule
\cellcolor{orange!30}\textbf{8-bit/16-bit}  & \cellcolor{orange!30}1.5/1.1$^\Diamond$&\cellcolor{orange!30}N.A. & \cellcolor{orange!30}29.50 &\cellcolor{orange!30}21.48 &\cellcolor{orange!30}19.14& \cellcolor{orange!30}17.89 &\cellcolor{orange!30}17.04 &\cellcolor{orange!30}16.37 &\cellcolor{orange!30}15.86 &\cellcolor{orange!30}15.48 &\cellcolor{orange!30}15.10 \\ \midrule
\textbf{16-bit} & N.A. & N.A. & 29.10 &21.35 &19.08 &17.84 &16.98 &16.33 &15.82 &15.44 &15.06 \\
\bottomrule
\end{tabular}}
{\footnotesize
\begin{tablenotes}
    \item[$\Diamond$] We use the two-stage threshold in this training task, see Appendix \ref{section: Experimental set up_ds} for more details.
\end{tablenotes}}
\vspace{-2.5mm}
\end{threeparttable}
\end{table*}

\begin{table}[t]
\caption{Trigger sparsity in the activation-quantization stress bench: proportion of eligible linear blocks exceeding GNMR and \(\Delta\)-GNMR thresholds.}
\vspace{-1.5mm}
\label{table:activation quantilization_ratio}
\centering
\setlength{\tabcolsep}{8pt}
\begin{threeparttable}
{\small
\begin{tabular}{ccccc} 
\toprule
\textbf{Model size} & $\alpha_t$ & $\beta_t$ & GNMR & $\Delta$-GNMR\\ \midrule
\multirow{2}{*}{\textbf{60M}} & 2.0&1.5&0.070\%&0.081\%\\ 
 & 3.0&2.0&0.019\%&0.033\% \\ \midrule
\multirow{2}{*}{\textbf{350M}} & 2.0&1.5&0.141\%&0.067\%\\ 
 & 3.0&2.0&0.019\%&0.016\% \\ \bottomrule
\end{tabular}}
\end{threeparttable}
\end{table}

\begin{table}[t!]
\vspace{-2mm}
\caption{The per-step time ($\downarrow$) and the throughput ($\uparrow$) of the LLaMA-2 1.3B training task with different precision strategies on the same device.}
\vspace{-1.5mm}
\label{table:llama2_activation_quant_speed}
\centering
\setlength{\tabcolsep}{6pt}
\begin{threeparttable}
{\small
\begin{tabular}{ccc} 
\toprule
\textbf{precision} & per-step time (s) & throughput (token/s) \\ \midrule
\textbf{8-bit} & 1.266 & 103532 \\ \midrule
\cellcolor{orange!30}\textbf{8-bit/16-bit}  & \cellcolor{orange!30}\textbf{1.264} &\cellcolor{orange!30}\textbf{103696}  \\ \midrule
\textbf{16-bit} & 1.294 & 101292\\ \bottomrule
\end{tabular}}
\end{threeparttable}
\vspace{-2mm}
\end{table}

\begin{table}[t!]
\caption{The validation perplexity ($\downarrow$) for the pre-training LLaMA-2 130M model with different two-stage GNMR threshold, $\alpha_t^{\text{init}}$ represents the threshold in the initial $2.5\%$ steps and $\alpha_t^{\text{main}}$ represents the threshold in the other steps.}
\vspace{-1.5mm}
\label{table:llama2_threshold selection}
\centering
\setlength{\tabcolsep}{10pt}
\begin{threeparttable}
{\small
\begin{tabular}{cc|ccc} 
\toprule
\multirow{2}{*}{$\alpha_t^{\text{init}}$}& \multirow{2}{*}{$\alpha_t^{\text{main}}$} & \multicolumn{3}{c}{Training tokens (B)} \\
\cline{3-5}
&& 1.3 & 2.1 & 2.6 \\ \midrule
1.1 & 1.1 & 26.53 &23.95 &23.57 \\
1.5 & 1.1 & \textbf{26.49} &\textbf{23.92} &\textbf{23.54} \\
1.5 & 1.5 & 26.56 &23.98 &23.60 \\
\bottomrule
\end{tabular}}
\end{threeparttable}
\vspace{-2mm}
\end{table}

\begin{table}[t!]
\caption{The validation perplexity ($\downarrow$) for the pre-training LLaMA-2 130M model with different $\mathrm{maxO}$.}
\vspace{-1.5mm}
\label{table:llama2_threshold maxo}
\centering
\setlength{\tabcolsep}{13pt}
\begin{threeparttable}
{\small
\begin{tabular}{c|ccc} 
\toprule
\multirow{2}{*}{$\mathrm{maxO}$} & \multicolumn{3}{c}{Training tokens (B)} \\
\cline{2-4}
& 1.3 & 2.1 & 2.6 \\ \midrule
2 & 26.56 &23.97 &23.59 \\
4 & \textbf{26.49} &\textbf{23.92} &\textbf{23.54} \\
6 & 26.50 &23.94 &23.56 \\
8 & 26.56 &24.01 &23.63 \\
\bottomrule
\end{tabular}}
\end{threeparttable}
\end{table}

\noindent Theorem~\ref{thm:GNMR-fixed} is diagnostic result.
The first term captures current-step deviation, while the second term captures historical-mean uncertainty and shrinks with history length. 
Thus, fixed high thresholds can become conservative after the historical mean stabilizes, motivating scheduled relaxation of the threshold \(\alpha_t\) and \(\beta_t\).

\section{Experiments}
\label{section: experiments}

We evaluate GNMR as a runtime stability controller in three settings:
activation-quantization stress, DeepSeek-style mixed-precision pre-training,
and LLaMA-2 13B fine-tuning stress. Appendix~\ref{section: Experimental set up}
reports full configurations, matched-trigger comparisons, threshold and
\(\mathrm{maxO}\) ablations, an ADAPT-style static baseline
\citep{menon2018adapt,Kummer_2023}, and MoE/PanGu-style stress tests.

\subsection{Pre-training LLaMA-2 under Activation-Quantization}

\noindent\textbf{Experiment setup.}
We pre-train LLaMA-2~\citep{touvron2023llama} models on C4-en~\citep{raffel2020exploring}, following the settings in~\citep{zhao2024galore,kong2025cr}; model sizes, token budgets, learning rates, and schedules are in Appendix~\ref{section: Experimental set up_aq} and Table~\ref{tab:model-config_aq}. This benchmark keeps weights and optimizer states in BF16 and applies token-wise quantization~\citep{jiang2022back,ramasinghe2025protocol} only to saved activations of LLaMA attention and SwiGLU MLP projection operators. These projection operators define the recoverable set, while embeddings, normalization, and output heads remain outside the GNMR action space. GNMR+\(\Delta\)-GNMR routes each recoverable unit between the low-cost activation path and the higher-fidelity recovery path, isolating runtime recovery from changes to the optimizer, model, or data.

\noindent\textbf{Experiment results.}
Table~\ref{table:activation quantilization} exposes the failure mode targeted by GNMR: a static low-cost activation path can severely degrade or fail, whereas GNMR-controlled 4-bit/8-bit recovery remains close to the fixed 8-bit and BF16/16-bit references across model sizes. This shows that online recovery over the eligible projection set can preserve the higher-fidelity trajectory without changing the training recipe.
Table~\ref{table:activation quantilization_ratio} shows that the recovery signal is highly localized: threshold exceedance stays below \(0.2\%\) in the reported runs. Thus the quality gain is not obtained by keeping a broad static slice of the model in high precision; GNMR acts as a selective runtime controller that converts rare unit-level risks into bounded recovery actions. Appendix~\ref{app:baseline-comparison-130m} compares this online policy with an ADAPT-style static sensitivity baseline~\citep{menon2018adapt,Kummer_2023}, and Appendix~\ref{section: Experimental set up} reports matched-trigger comparisons under the same recovery actuator.

\begin{table*}[t]
\centering
\small
\caption{
LLaMA-2 13B fine-tuning stress and downstream evaluation.
Higher is better for GSM8K, MMLU, and HellaSwag; lower is better for WikiText-2 PPL.
}
\label{tab:13b_downstream}
\begin{tabular}{lcccc}
\toprule
Method & GSM8K EM & MMLU Acc. & HellaSwag Acc. & WikiText-2 PPL \\
\midrule
BF16 & \textbf{15.09} & \textbf{29.92} & 64.62 & \textbf{5.4503} \\
Fixed INT8 & 14.71 & 29.88 & 64.55 & 5.4519 \\
GNMR+\(\Delta\) INT8/BF16 & 15.01 & \textbf{29.92} & \textbf{64.67} & 5.4511 \\
\bottomrule
\end{tabular}
\end{table*}

\subsection{Pre-training LLaMA-2 with DeepSeek-style Mixed Precision}

\noindent\textbf{Experiment setup.}
We next evaluate GNMR on C4-en under a DeepSeek-V3-style precision hierarchy
\citep{liu2024deepseek}. Operators fixed to 16-bit or 32-bit by the recipe stay
unchanged; GNMR controls only the remaining low-precision-eligible subgraph.
Because operator-level recovery can create inconsistent execution inside a
block, we use each transformer block as the recoverable unit. The active set is
capped by \(\mathrm{maxO}=6\), so at most \(25\%\) of layers use the recovery
path at a step. We compare dynamic 6-bit/8-bit against fixed 8-bit, and dynamic
8-bit/16-bit against fixed 16-bit. Details are in
Appendix~\ref{section: Experimental set up_ds} and
Table~\ref{tab:model-config_ds}.

\noindent\textbf{Experiment results.}
Table~\ref{table:llama2_activation_quant} shows that GNMR stays close to the
corresponding higher-fidelity references throughout the DeepSeek-style recipe.
This is recipe-level integration evidence: GNMR controls temporary recovery decisions over its eligible subgraph.
Table~\ref{table:llama2_activation_quant_speed} characterizes the reported
1.3B training stack. Dynamic 8-bit/16-bit preserves the low-cost execution
profile under this implementation while following the higher-fidelity quality
trend. Figures~\ref{fig:gnmr_bf16} and~\ref{fig:gnmr_duibi} also give diagnostic
risk traces, showing that GNMR-controlled recovery dampens high-risk excursions
relative to the fixed low-cost path.

\noindent\textbf{Scaling, threshold, and budget behavior.}
Figure~\ref{fig:gnmr_llama_scaling} reports GNMR-controlled LLaMA-style
pre-training from 130M to 3B under the 6-bit/8-bit setting; the corresponding
configurations and token budgets are documented in
Appendix~\ref{section: Experimental set up_ds}. This result shows that GNMR preserves the expected model-and-token scaling trend in the tested regime.
Tables~\ref{table:llama2_threshold selection} and
\ref{table:llama2_threshold maxo} study threshold scheduling and the active
recovery budget. The two-stage threshold reflects noisy early GNMR history and
more stable later running statistics. The \(\mathrm{maxO}\) study characterizes a recovery--overhead trade-off: \(\mathrm{maxO}\) bounds active recovery actions and peak overhead.

\begin{figure}[t]
	\centering
	\begin{minipage}[c]{0.464\textwidth}
    \centering
    	\includegraphics[width=1\textwidth]{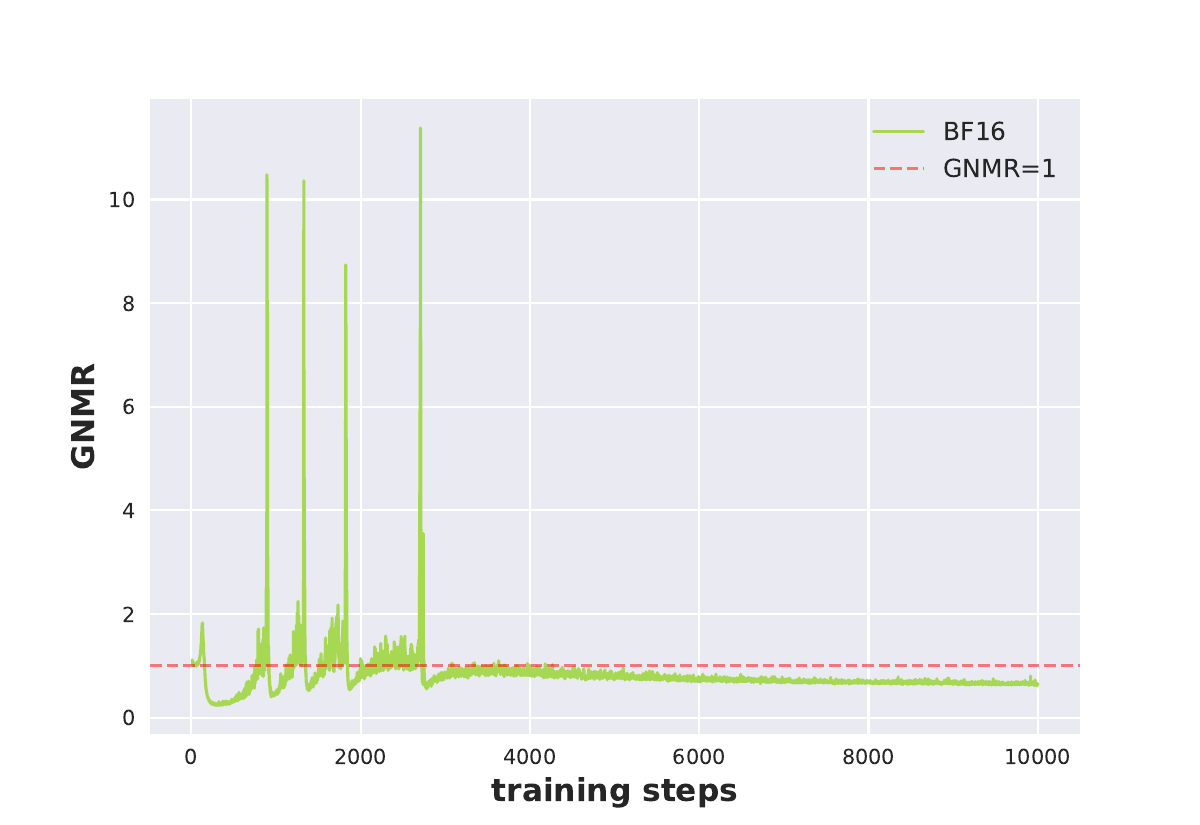}
    \caption{The GNMR curve during the pre-training task of LLaMA-2 model with BF16 precision.}
    \label{fig:gnmr_bf16}
	\end{minipage} 
    \hspace{2mm}
	\begin{minipage}[c]{0.464\textwidth}
    \centering
    	\includegraphics[width=1\textwidth]{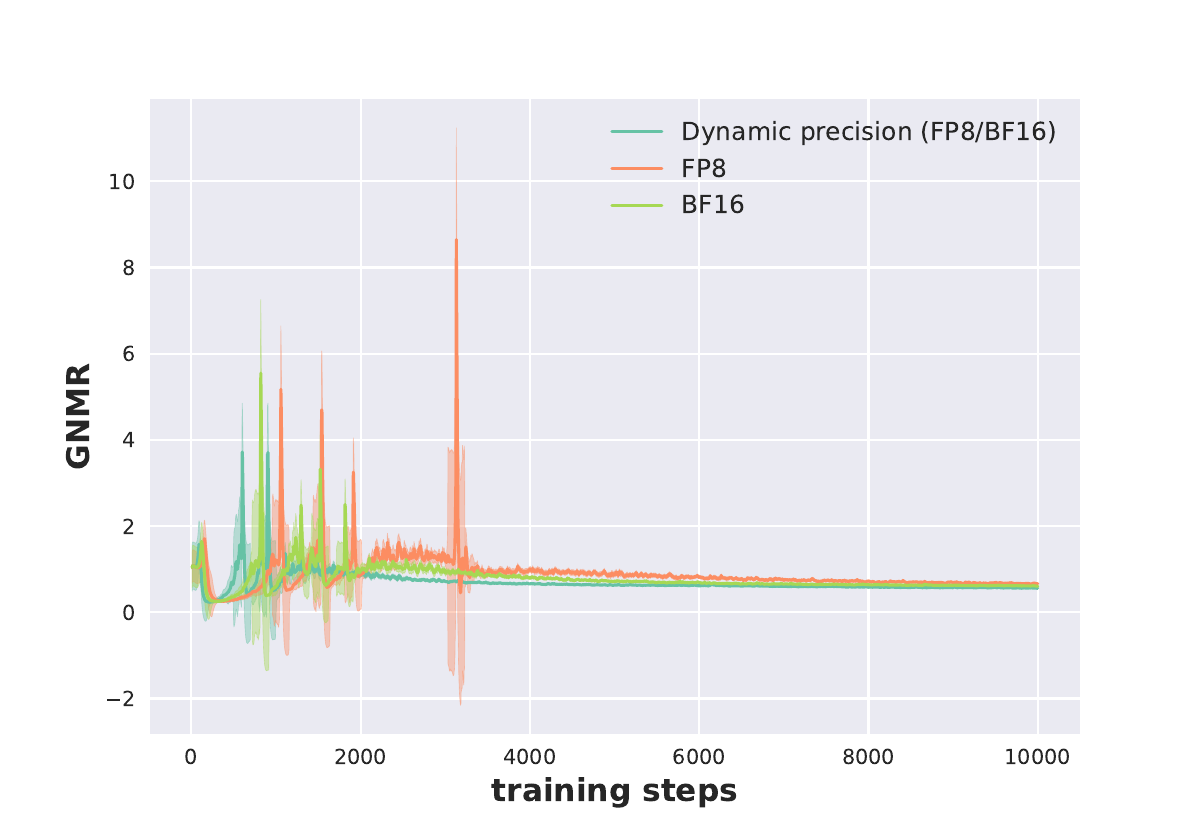}
    \caption{Evolution of average GNMR values across operators under different precision strategies.}
    \label{fig:gnmr_duibi}
	\end{minipage} 
\end{figure}

\begin{figure}[t]
\centering
\includegraphics[width=0.65\textwidth]{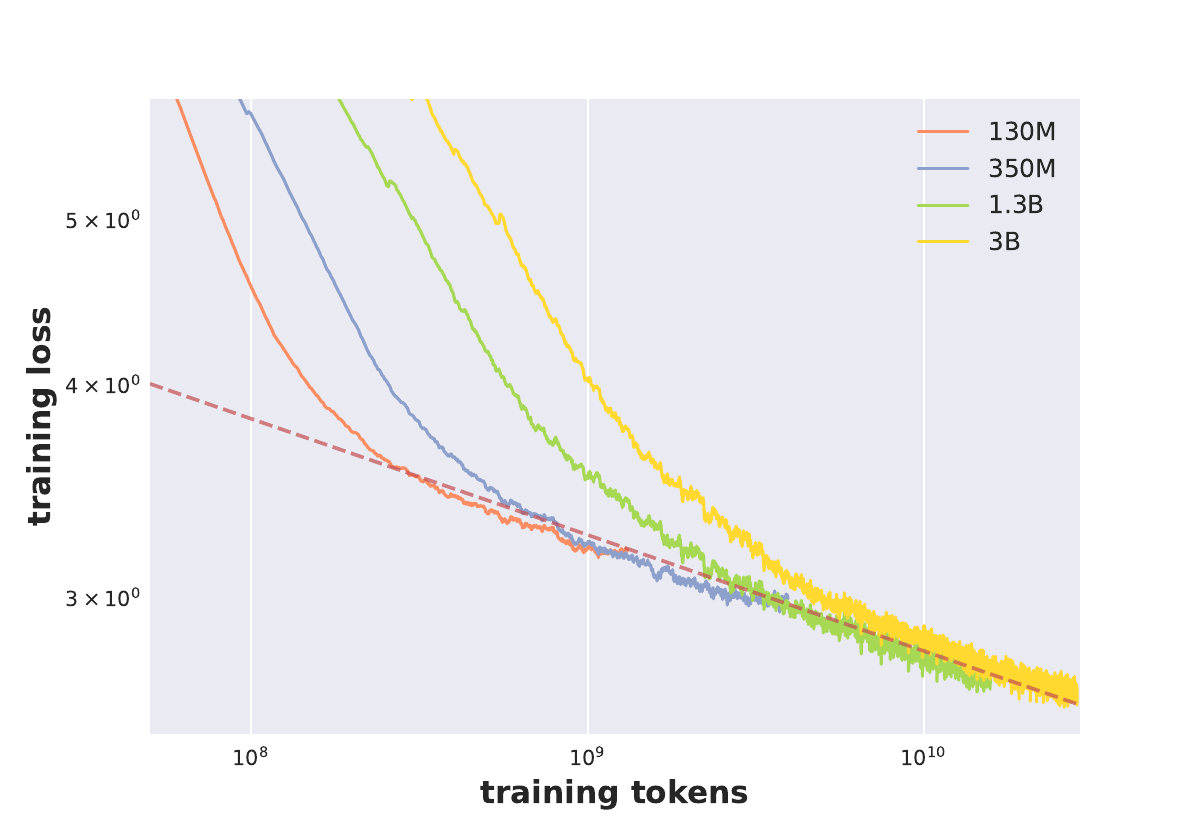}
\caption{Scaling trajectory for GNMR-controlled LLaMA-style pre-training under the 6/8-bit low-cost/recovery setting. The corresponding model configurations, token budgets, learning rates, and final evaluation PPLs are reported in Table~\ref{tab:model-config_ds}.}
\vspace{-1mm}
\label{fig:gnmr_llama_scaling}
\end{figure}

\subsection{LLaMA-2 13B Fine-tuning Stress Test}

\noindent\textbf{Experiment setup.}
Finally, we evaluate LLaMA-2 13B under BF16, fixed INT8, and dynamic INT8/BF16 controlled by GNMR+\(\Delta\)-GNMR. 
The controller acts only on the low-precision-eligible recovery path, with \(\mathrm{maxO}=10\) over 40 decoder layers. 
We fine-tune LoRA~\citep{hu2022lora} adapters while keeping the base model frozen, and evaluate matched downstream tasks covering math reasoning, multiple-choice knowledge, commonsense completion, and language modeling: GSM8K~\citep{cobbe2021training}, MMLU~\citep{hendrycks2020measuring}, HellaSwag~\citep{zellers2019hellaswag}, and WikiText-2~\citep{merity2016pointer}.

\noindent\textbf{Experiment results.}
Table~\ref{tab:13b_downstream} reports matched downstream metrics after the 13B fine-tuning stress test. 
GNMR+\(\Delta\) INT8/BF16 obtains high-fidelity-level downstream performance across all reported tasks, matching BF16 on MMLU and outperforming BF16 on HellaSwag. 
On GSM8K and WikiText-2, GNMR+\(\Delta\) remains close to the BF16 reference while operating through the low-cost INT8/BF16 recovery interface. 
These results support GNMR as a runtime controller that preserves downstream behavior under a low-cost INT8 path. 
Additional fine-tuning and evaluation details are provided in Appendix~\ref{section: Experimental set up_13b}.

\section{Conclusion}
This paper frames low-precision language model training stability as a runtime control problem and presents GNMR as a controller for this setting. GNMR and $\Delta$-GNMR estimate long-horizon and short-window operator-level risks, then allocate a limited recovery budget to the most risky recoverable units. Experiments across activation-quantization stress, TE-backed low-precision training, and large-model post-training stress show that GNMR improves stability while preserving low-cost execution.

\section{Limitations}
GNMR targets runtime stability control over backend-provided low-cost and recovery paths. Its effectiveness depends on backend-level observability, recoverable units, and compatible recovery implementations. Further engineering can reduce casting, routing, and recovery-switching overheads.

\newpage

{
\small

\bibliography{reference}
\bibliographystyle{icml2026}
}

\newpage
\appendix
\begin{center}
{\sffamily\bfseries\fontsize{15}{18}\selectfont
    Appendix\par}
\end{center}
\vspace{5mm}

\section{Related works}
\label{sec: appendix_related_works}

This section reviews work on low-precision training and adaptation, with emphasis on how prior methods relate to the runtime stability-control view studied in this paper.

\medskip\noindent\textbf{Mixed-precision training and numerical formats.}
Mixed-precision training reduces memory traffic and improves accelerator throughput by combining low-cost arithmetic with higher-fidelity accumulation or safeguards.
Early FP16 training used loss scaling and selective FP32 accumulation to preserve accuracy on modern accelerators~\cite{micikevicius2017mixed}, while BF16 improved dynamic-range robustness and became a common choice for large-scale language model training~\cite{wangbfloat16}.
More recent work has moved toward sub-16-bit training.
FP8 formats and recipes, including E4M3/E5M2 variants, have shown that reduced-precision training can converge under carefully designed scaling rules~\cite{micikevicius2022fp8,peng2023fp8}.
Transformer Engine (TE) further makes FP8 execution practical through delayed scaling and per-tensor format selection~\cite{perez2023traininginferencelargelanguage}.
COAT compresses optimizer states and activations to FP8~\cite{xi2024coat}, while recent FP4 and NVFP4 training studies push precision even lower with stronger scaling, outlier handling, and selective high-fidelity components~\cite{wang2025optimizing,castro2025quartet,abecassis2025pretraining}.
These works provide important backend execution paths; GNMR instead studies when an existing low-cost path should be temporarily recovered for stability.

\medskip\noindent\textbf{Quantization, outliers, and memory-efficient adaptation.}
Quantization methods for language models highlight that low precision is often limited by outliers and fragile layers.
LLM.int8 and GPT3.int8 identify outlier features as a key challenge for accurate low-bit execution~\cite{dettmers2022gpt3}, while SmoothQuant shifts activation difficulty into weights to enable W8A8 inference~\cite{xiao2023smoothquant}.
GPTQ, AWQ, and ZeroQuant further improve post-training quantization through layer-wise reconstruction, salient-weight protection, or distillation~\cite{frantar2022gptq,lin2024awq,yao2022zeroquant}.
For adaptation, QLoRA trains adapters over frozen 4-bit bases using NF4 quantization and paged optimizers~\cite{dettmers2023qlora}.
Quantization-aware training and efficient QAT variants further reduce the quality gap at low bit-widths~\cite{liu2023llmqatdatafreequantizationaware,bondarenko2024low,chen2024efficientqat,ashkboos2024efqat,zhao2023pytorch}.
Outlier suppression and smoothing methods also show that precision risk is often local rather than uniform~\cite{wei2023outliersuppressionaccuratequantization}.
These methods mainly reduce quantization error through calibration, smoothing, or training procedures; GNMR uses runtime training signals to decide where recovery is needed during training.

\medskip\noindent\textbf{Systems support for low-precision training.}
Large-scale language model training relies on systems that combine low-precision execution with distributed parallelism.
Megatron-LM scales transformer training through tensor and pipeline parallelism~\cite{shoeybi2019megatron}, while ZeRO and ZeRO++ reduce optimizer-state and communication overhead through sharding and communication-aware design~\cite{rajbhandari2020zero,wang2023zero++}.
Communication compression and quantization further reduce distributed training cost, including 4-bit communication for sharded training and protocol-level compression~\cite{jia2024sdp4bit4bitcommunicationquantization,ramasinghe2025protocol,kong2025clapping}.
Pipeline and offload systems also reduce memory and communication bottlenecks in large training runs~\cite{huang2019gpipe,wan2025pipeoffload}.
These systems are complementary to our work: they determine how low-cost execution is implemented and scaled, while GNMR determines when local runtime evidence justifies temporary recovery.

\medskip\noindent\textbf{Training stability signals and optimizer-level stabilization.}
Stability failures such as loss spikes and gradient spikes are a known challenge in large-scale pre-training~\cite{takase2023spike}.
Classical gradient clipping controls exploding gradients~\cite{pascanu2013difficulty}, and Adaptive Gradient Clipping normalizes clipping by parameter scale~\cite{brock2021high}.
Transformer stabilization methods such as Pre-LN, RMSNorm, and DeepNorm help regulate update magnitudes in deep networks~\cite{xiong2020layernormalizationtransformerarchitecture}.
Gradient-noise-scale analysis characterizes batch-size regimes~\cite{mccandlish2018empirical}, while GSNR-style signals use gradient statistics to guide large-batch training~\cite{jiang2023accelerating}.
Recent spike-aware optimizers reset or adjust updates when instability is detected~\cite{huang2025spam,huang2025stable}.
Reduced-precision training can further increase sensitivity to stability and hyperparameters~\cite{lee2024fp8}.
These works motivate the use of runtime stability signals, but they do not directly provide an operator-local controller for budgeted precision recovery.

\medskip\noindent\textbf{Dynamic and operator-wise precision selection.}
Dynamic precision methods adapt numerical precision over time, layer, or operator.
ADAPT-style approaches estimate sensitivity and assign precision based on profiling or program analysis~\cite{menon2018adapt,Kummer_2023}.
Learnable Dynamic Precision formulates temporal and spatial precision choices as learnable parameters to optimize the compute-accuracy tradeoff~\cite{yu2022ldplearnabledynamicprecision}.
Convergence-aware operator-wise mixed precision methods also explore precision assignment using convergence or sensitivity information~\cite{dai2025convergence}.
These methods are closely related to our setting, but they are typically offline, coarse-grained, or aimed at compute-accuracy assignment rather than runtime stability control.
GNMR differs by using online operator-wise gradient deviations to trigger temporary recovery under an explicit budget, aligning precision intervention with local training risk rather than a fixed schedule or static sensitivity profile.

\section{Additional details for the operator-level modeling.}
\label{appendix: operator-level modeling}

This section complements the operator-wise runtime monitoring setup in Section~\ref{sec:preliminaries}. 
The main text keeps only the notation needed to define local gradient statistics for GNMR. 
Here, we spell out the sibling-operator, aggregation, and backpropagation details for common Transformer components. 
The goal is not to introduce a new architectural abstraction, but to make explicit how attention, feed-forward modules, normalization, residual paths, and MoE layers map to the monitored units used by the controller. 
The discussion applies to both encoder-only and decoder-style models \cite{vaswani2017attention,brown2020language}.

\medskip\noindent\textbf{Nest--parallel operator details.}
A depth-$U$ Transformer stack corresponds to the composition of nested operators \(\Nest^{[1]},\ldots,\Nest^{[U]}\) acting across layers, as in Eq.~\eqref{eq: modeling_2}. 
Inside one block \(u\), sibling operators act on the same input:
\begin{align}
\label{eq: modeling_3}
z^{[u,k]} := f^{[u,k]}\!\left(y^{[u-1]};\,\theta^{[u,k]}\right),\quad k=1,\dots,K_u.
\end{align}
An aggregation operation then merges their outputs and may also use the skip input \(y^{[u-1]}\):
\begin{equation}
\label{eq: modeling_4}
\begin{aligned}
&\Nest^{[u]}=z^{[u,0]}\!\left(z^{[u,1]},\ldots,z^{[u,K_u]};\ \theta^{[u,0]},\,y^{[u-1]}\right),\quad\text{with}\quad\Xi^{[u]}=\left(\theta^{[u,k]}\right)_{k=0}^{K_u}.
\end{aligned}
\end{equation}
This pattern captures multi-head attention, feed-forward modules, residual connections, normalization, and MoE variants. 
It also provides a common index for the monitored units whose gradient norms are used by GNMR.

\medskip\noindent\emph{Multi-head self-attention (MHA).}
Let \(H\) be the layer input. The head-wise projections are parallel operators
\begin{align}
Q &= H W_Q^{[u]},\quad
K = H W_K^{[u]},\quad
V = H W_V^{[u]} .
\end{align}
The attention aggregator then maps \(\left(Q,K,V\right)\) and the residual input \(H\) to the layer output:
\begin{equation}
\begin{aligned}
\label{eq:attention}
z_{\mathrm{att}}^{[u]}\!\left(Q,K,V;\,W_O^{[u]},\,H\right)=\mathrm{softmax}\!\left(\frac{QK^\top}{\sqrt{d_k}}\right) V\, W_O^{[u]} + H .
\end{aligned}
\end{equation}
Stacking multiple heads fits by letting the sibling set contain per-head projections and attention kernels, with the aggregator concatenating head outputs and applying \(W_O^{[u]}\).

\medskip\noindent\emph{Feed-forward networks (FFN) and MoE.}
A standard two-layer FFN with residual can be represented with \(K_u=1\) as
\begin{align}
\label{eq:ffn}
z^{[u,0]}\!\left(z^{[u,1]};\,\theta^{[u,0]},\,y\right)
= y + W_2\,\varphi\!\left(W_1\,\mathrm{LN}(y)\right).
\end{align}
Mixture-of-experts (MoE) variants can be represented using \(K_u>1\) sibling experts \(f^{[u,k]}\) and an aggregator that gates or mixes expert outputs, e.g.,
\begin{equation}
\begin{aligned}
&z^{[u,0]}\!\left(z^{[u,1]},\ldots,z^{[u,K_u]};\,\theta^{[u,0]},\,y\right)= \sum_{k=1}^{K_u} \alpha_k(y)\,z^{[u,k]},
\end{aligned}
\end{equation}
where \(\sum_{k=1}^{K_u}\alpha_k(y)=1\). 
This form covers router-based and top-\(k\) mixing used in modern MoE systems \cite{du2022glam,fedus2022switch,dai2024deepseekmoe,hu2026synergistic}.
Variants with concatenation followed by projection are obtained by letting \(z^{[u,0]}\) first concatenate \(\left(z^{[u,1]},\ldots,z^{[u,K_u]}\right)\) and then apply a linear map.

\medskip\noindent\textbf{Backpropagation in the operator view.}
Denote the upstream gradient at layer \(u\) by \(\bar y^{[u]} := \partial L/\partial y^{[u]}\), and define \(\bar z^{[u,k]} := \partial L/\partial z^{[u,k]}\). 
The aggregator routes sensitivity to each sibling output via
\begin{align}
\bar z^{[u,k]}
= \left(\frac{\partial z^{[u,0]}}{\partial z^{[u,k]}}\right)^{\!\top}\!\bar y^{[u]},
\quad k=1,\ldots,K_u .
\end{align}
Parameter gradients follow directly from the local Jacobians:
\begin{equation}
\begin{aligned}
    \nabla_{\theta^{[u,k]}} L
= \left(\frac{\partial f^{[u,k]}}{\partial \theta^{[u,k]}}\right)^{\!\top}\!\bar z^{[u,k]},
\quad
\nabla_{\theta^{[u,0]}} L
= \left(\frac{\partial z^{[u,0]}}{\partial \theta^{[u,0]}}\right)^{\!\top}\!\bar y^{[u]} .
\end{aligned}
\end{equation}

The gradient w.r.t.\ the layer input \(y^{[u-1]}\) sums the sibling and aggregator contributions:
\begin{align}
\bar y^{[u-1]}
= \sum_{k=1}^{K_u}\left(\frac{\partial f^{[u,k]}}{\partial y}\right)^{\!\top}\!\bar z^{[u,k]}
\;+\;
\left(\frac{\partial z^{[u,0]}}{\partial y}\right)^{\!\top}\!\bar y^{[u]} .
\end{align}
These expressions are shape-consistent with Eq.~\eqref{eq: modeling_4} and specialize to standard Transformer derivatives: for attention, the aggregator derivatives backpropagate through the \(\mathrm{softmax}\) and output projection; for FFN/MoE, they reduce to linear maps and elementwise nonlinearity Jacobians.

\medskip\noindent\textbf{Implementation scope.}
Residual paths are naturally included by passing \(y^{[u-1]}\) to the aggregator, as in Eqs.~\eqref{eq:attention}--\eqref{eq:ffn}. 
Layer normalization can be modeled either as a sibling \(f^{[u,k]}\) or inside the aggregator, depending on implementation needs. 
Parameter sharing, masking, and causal structure can be represented as operators inside \(f^{[u,k]}\) or the aggregator.

\medskip
In summary, the nest--parallel operator view recovers standard Transformer components while providing a common indexing scheme for operator-wise gradient statistics. 
This makes it a convenient interface for GNMR-based runtime stability monitoring and budgeted recovery.

\section{Assumptions and detailed proofs}
\label{appendix: assumptions_and_proofs}

This appendix provides the standing assumptions and detailed proofs for
Theorems~\ref{thm:op-descent-gain-main} and~\ref{thm:GNMR-fixed}.
The first part proves the bound-level gain of budgeted selective recovery.
The second part proves the fixed-threshold GNMR concentration result.
Throughout, \(\mathcal B_{\mathrm{rec}}\) is a finite set of recoverable units,
\(\mathcal F_s\) denotes the history available before the update at step \(s\),
and the active recovery set \(\mathcal A_s\) is \(\mathcal F_s\)-measurable.
When newly selected units at step \(t\) are denoted by \(\mathcal S_t\), we use
the lock/update convention in Theorem~\ref{thm:op-descent-gain-main}, namely
\(\mathcal S_t\subseteq\mathcal A_{t+1}\).

\subsection{Assumptions}
\label{appendix: assumptions}

\begin{assumption}[Smoothness]
\label{ass:smooth-partial-NP}
The training objective \(\mathcal L(\theta)\) is \(L\)-smooth, i.e., for any
\(\theta,\theta'\),
\begin{equation}
\label{eq:smoothness-definition}
\begin{aligned}
\mathcal L(\theta')
&\le
\mathcal L(\theta)
+
\left\langle \nabla\mathcal L(\theta),\theta'-\theta\right\rangle+
\frac{L}{2}\|\theta'-\theta\|^2 .
\end{aligned}
\end{equation}
\end{assumption}

\begin{assumption}[SGD noise]
\label{ass:sgd_noise}
The ideal mini-batch stochastic gradient at step \(s\) can be written as
\begin{align}
\label{eq:sgd-noise-decomposition}
g_s &= \nabla\mathcal L(\theta_s)+\xi_s,
\end{align}
where
\begin{align}
\label{eq:sgd-noise-moments}
\mathbb E[\xi_s\mid\mathcal F_s] &=0,
\qquad
\mathbb E[\|\xi_s\|^2\mid\mathcal F_s]\le \sigma^2 .
\end{align}
\end{assumption}

\begin{assumption}[Perturbation envelope]
\label{ass:perturbation_envelope}
For each recoverable unit \(b\in\mathcal B_{\mathrm{rec}}\), and for each path
\(q\in\{\mathrm{low},\mathrm{rec}\}\), let \(\mathcal E_q^b(\rho_s^b)\) denote
an operator-wise upper-envelope penalty contributed by unit \(b\) to the
one-step smoothness upper bound at risk state \(\rho_s^b\).  For an active set
\(\mathcal A_s\), let \(q_b(\mathcal A_s)=\mathrm{rec}\) if
\(b\in\mathcal A_s\), and let \(q_b(\mathcal A_s)=\mathrm{low}\) otherwise.
If \(r_s(\mathcal A_s)\) is the path-dependent perturbation in the actual update
direction, then the conditional path-dependent remainder in the smoothness
bound satisfies
\begin{equation}
\label{eq:perturbation-envelope}
\begin{aligned}
&\mathbb E\!\left[
-\eta\left\langle \nabla\mathcal L(\theta_s),r_s(\mathcal A_s)\right\rangle
+L\eta^2\left\langle g_s,r_s(\mathcal A_s)\right\rangle+\frac{L\eta^2}{2}\|r_s(\mathcal A_s)\|^2
\middle|\mathcal F_s\right]
\\
\le&
\sum_{b\in\mathcal B_{\mathrm{rec}}}
\Big[
\mathbf 1\{b\in\mathcal A_s\}
\mathcal E_{\mathrm{rec}}^b(\rho_s^b)+
\mathbf 1\{b\notin\mathcal A_s\}
\mathcal E_{\mathrm{low}}^b(\rho_s^b)
\Big].
\end{aligned}
\end{equation}
We use the right-hand side of \eqref{eq:perturbation-envelope} to define the
path-dependent part of the surrogate one-step upper bound.  The recovery gap is
\begin{align}
\label{eq:gap-definition}
\Delta\mathcal E^b(\rho)
&=
\mathcal E_{\mathrm{low}}^b(\rho)
-
\mathcal E_{\mathrm{rec}}^b(\rho).
\end{align}
The gap in \eqref{eq:gap-definition} is assumed to satisfy the lower-envelope
condition in Assumption~\ref{ass:persist} on the risky regimes used in
Theorem~\ref{thm:op-descent-gain-main}.
\end{assumption}

\begin{remark}[One sufficient parameterization]
\label{rem:penalty_parameterization}
Assumption~\ref{ass:perturbation_envelope} does not require low-cost errors to
be independent or zero-mean.  One sufficient parameterization is as follows.
Let \(e_{b,s}^q\) be the effective perturbation induced by the precision
transformation of unit \(b\) under path \(q\), after it propagates through the
relevant forward, activation-storage, or backward computation.  Define
\begin{equation}
\label{eq:penalty-moment-definitions}
\begin{aligned}
\mu_q^b(\rho)
&=
\left\|
\mathbb E[e_{b,s}^{q}\mid\mathcal F_s,\rho_s^b=\rho]
\right\|_2,\quad
v_q^b(\rho)
=
\mathbb E\!\left[
\|e_{b,s}^{q}\|_2^2
\mid
\mathcal F_s,\rho_s^b=\rho
\right].
\end{aligned}
\end{equation}
If \(G_s\ge \|\nabla\mathcal L(\theta_s)\|\), and \(\Gamma_b\ge1\) absorbs cross
terms, forward/backward propagation, residual or normalization amplification,
and operator coupling, then one may take
\begin{align}
\label{eq:penalty-parameterization}
\mathcal E_q^b(\rho)
&=
\eta G_s\mu_q^b(\rho)
+
\frac{L\eta^2}{2}\Gamma_b v_q^b(\rho).
\end{align}
Indeed, by \eqref{eq:penalty-moment-definitions}, the linear term in the
smoothness remainder is bounded by
\begin{equation}
\label{eq:linear-envelope-bound}
\begin{aligned}
&\eta\|\nabla\mathcal L(\theta_s)\|
\left\|
\mathbb E[e_{b,s}^q\mid\mathcal F_s,\rho_s^b=\rho]
\right\|_2\le
\eta G_s\mu_q^b(\rho).
\end{aligned}
\end{equation}
The quadratic term and the mixed term
\(L\eta^2\langle g_s,r_s\rangle\) can be bounded by a shared path-independent
term plus operator-wise second-moment penalties using Cauchy--Schwarz and Young
inequalities; the constants and coupling losses are collected into
\(\Gamma_b\) in \eqref{eq:penalty-parameterization}.  Thus the envelope can
include bias, variance, activation quantization error, forward/backward
perturbations, and backend-specific scaling or casting effects.
\end{remark}

\begin{assumption}[Selected-risk persistence]
\label{ass:persist}
Let \(\mathcal S_t\) be the newly selected recovery set at step \(t\).  For each
\(b\in\mathcal B_{\mathrm{rec}}\), assume that there exist
\(\pi_b\in(0,1]\) and \(\tau_b\) such that, whenever
\(\Pr(b\in\mathcal S_t)>0\),
\begin{align}
\label{eq:selected-risk-persistence}
\Pr(\rho_{t+1}^b\ge \tau_b\mid b\in\mathcal S_t) &\ge \pi_b .
\end{align}
We also assume the following lower envelope for the perturbation-penalty gap in
\eqref{eq:gap-definition}:
\begin{equation}
\label{eq:gap-lower-envelope}
\begin{aligned}
\Delta\mathcal E^b(\rho)
&\ge
\Delta\mathcal E^b(\tau_b)\mathbf 1\{\rho\ge\tau_b\},\quad
\Delta\mathcal E^b(\tau_b)
\ge 0 .
\end{aligned}
\end{equation}
Equivalently, the recovery path is not worse than the low-cost path in the
regimes used by the lower-bound argument, and the gap is at least
\(\Delta\mathcal E^b(\tau_b)\) once the next-step risk exceeds \(\tau_b\).
\end{assumption}

This is a sufficient controller-efficacy condition rather than a universal distributional assumption; it can be estimated from training traces.

\begin{remark}[Budget convention]
\label{rem:maxO_budget}
In Theorem~\ref{thm:op-descent-gain-main}, \(\mathrm{maxO}\) is interpreted as
a budget on the active recovery set, i.e., \(|\mathcal A_s|\le\mathrm{maxO}\).
This matches the claim that \(\mathrm{maxO}\) controls instantaneous overhead.
If an implementation instead caps only newly admitted units, then the same proof
applies with \(|\mathcal S_t|\le\mathrm{maxO}\), but the runtime overhead should
be reported using the active set size \(|\mathcal A_s|\).
\end{remark}

\subsection{Proof of Theorem~\ref{thm:op-descent-gain-main}}
\label{sec:the_part2}

We prove the theorem through several lemmas.  The proof is stated at the level
of smoothness upper bounds.  Thus, the quantity
\begin{align}
\bar{\mathcal U}_s^{\mathrm{Low}}
-
\bar{\mathcal U}_s^{\mathrm{Rec}}(\mathcal A_s)
\end{align}
is a difference between two surrogate one-step upper bounds, not an exact
difference between actual losses.

\begin{lemma}[Smoothness upper bound with perturbation envelope]
\label{lem:smoothness_envelope}
For any active recovery set \(\mathcal A_s\subseteq\mathcal B_{\mathrm{rec}}\),
the surrogate one-step upper bound can be written as
\begin{equation}
\label{eq:smoothness-envelope-bound}
\begin{aligned}
\bar{\mathcal U}_s(\mathcal A_s)
&=
\mathcal U_{0,s}
+
\sum_{b\in\mathcal B_{\mathrm{rec}}}
\Big[
\mathbf 1\{b\in\mathcal A_s\}
\mathcal E_{\mathrm{rec}}^b(\rho_s^b)+
\mathbf 1\{b\notin\mathcal A_s\}
\mathcal E_{\mathrm{low}}^b(\rho_s^b)
\Big],
\end{aligned}
\end{equation}
where \(\mathcal U_{0,s}\) collects the ideal SGD terms and all terms shared by
the two paths.  Moreover,
\begin{align}
\label{eq:surrogate-upper-bound-property}
\mathbb E[\mathcal L(\theta_{s+1})\mid\mathcal F_s]
&\le
\bar{\mathcal U}_s(\mathcal A_s).
\end{align}
\end{lemma}

\begin{proof}
Let \(f_s:=\nabla\mathcal L(\theta_s)\), and let \(\widetilde g_s\) denote the
actual update direction under a given path assignment.  By
Assumption~\ref{ass:smooth-partial-NP}, for the update
\(\theta_{s+1}=\theta_s-\eta\widetilde g_s\), we have
\begin{equation}
\label{eq:smoothness-step-bound}
\begin{aligned}
\mathcal L(\theta_{s+1})
&\le
\mathcal L(\theta_s)
-
\eta\langle f_s,\widetilde g_s\rangle+
\frac{L\eta^2}{2}\|\widetilde g_s\|^2 .
\end{aligned}
\end{equation}
Write the actual update direction as the ideal mini-batch stochastic gradient
plus a path-dependent perturbation:
\begin{align}
\label{eq:actual-direction-decomposition}
\widetilde g_s
&=
g_s+r_s(\mathcal A_s),
\qquad
 g_s=f_s+\xi_s .
\end{align}
Substituting \eqref{eq:actual-direction-decomposition} into
\eqref{eq:smoothness-step-bound} gives
\begin{equation}
\label{eq:smoothness-expanded}
\begin{aligned}
\mathcal L(\theta_{s+1})
&\le
\mathcal L(\theta_s)
-
\eta\langle f_s,g_s\rangle
+
\frac{L\eta^2}{2}\|g_s\|^2-
\eta\langle f_s,r_s(\mathcal A_s)\rangle
+
L\eta^2\langle g_s,r_s(\mathcal A_s)\rangle+
\frac{L\eta^2}{2}\|r_s(\mathcal A_s)\|^2 .
\end{aligned}
\end{equation}
Taking conditional expectation with respect to \(\mathcal F_s\), the ideal part
in \eqref{eq:smoothness-expanded} is
\begin{equation}
\label{eq:ideal-part-bound}
\begin{aligned}
&\mathbb E\!\left[
\mathcal L(\theta_s)
-
\eta\langle f_s,g_s\rangle
+
\frac{L\eta^2}{2}\|g_s\|^2
\middle|\mathcal F_s\right]
=
\mathcal L(\theta_s)
-
\eta\|f_s\|^2
+
\frac{L\eta^2}{2}
\mathbb E[\|f_s+\xi_s\|^2\mid\mathcal F_s]
\\
=&
\mathcal L(\theta_s)
-
\eta\|f_s\|^2
+
\frac{L\eta^2}{2}
\left(
\|f_s\|^2
+
\mathbb E[\|\xi_s\|^2\mid\mathcal F_s]
\right)\le\mathcal L(\theta_s)
-
\eta\left(1-\frac{L\eta}{2}\right)\|f_s\|^2
+
\frac{L\eta^2}{2}\sigma^2.
\end{aligned}
\end{equation}
The first equality in \eqref{eq:ideal-part-bound} uses
\eqref{eq:sgd-noise-decomposition}, and the last inequality uses
\eqref{eq:sgd-noise-moments}.  Then, define the shared ideal term by
\begin{equation}
\label{eq:u0-definition}
\begin{aligned}
\mathcal U_{0,s}
&:=
\mathcal L(\theta_s)
-
\eta\left(1-\frac{L\eta}{2}\right)\|f_s\|^2+
\frac{L\eta^2}{2}\sigma^2 .
\end{aligned}
\end{equation}
The remaining terms in \eqref{eq:smoothness-expanded} are exactly the
path-dependent perturbation terms.  By Assumption~\ref{ass:perturbation_envelope},
their conditional contribution is bounded by \eqref{eq:perturbation-envelope}.
Combining \eqref{eq:u0-definition} and \eqref{eq:perturbation-envelope}
yields \eqref{eq:smoothness-envelope-bound}, and hence
\eqref{eq:surrogate-upper-bound-property}.
\end{proof}

\begin{lemma}[Active-set bound-level gain]
\label{lem:active_set_gain}
Let \(\bar{\mathcal U}_s^{\mathrm{Low}}\) be the upper bound when every
recoverable unit uses the low-cost path, and let
\(\bar{\mathcal U}_s^{\mathrm{Rec}}(\mathcal A_s)\) be the upper bound under
active recovery set \(\mathcal A_s\).  Then it holds that:
\begin{equation}
\label{eq:bound_gain_identity}
\begin{aligned}
&\bar{\mathcal U}_s^{\mathrm{Low}}
-
\bar{\mathcal U}_s^{\mathrm{Rec}}(\mathcal A_s)=
\sum_{b\in\mathcal B_{\mathrm{rec}}}
\Delta\mathcal E^b(\rho_s^b)
\mathbf 1\{b\in\mathcal A_s\}.
\end{aligned}
\end{equation}
\end{lemma}

\begin{proof}
By Lemma~\ref{lem:smoothness_envelope}, the always-low-cost upper bound is
\begin{align}
\label{eq:always-low-upper-bound}
\bar{\mathcal U}_s^{\mathrm{Low}}
&=
\mathcal U_{0,s}
+
\sum_{b\in\mathcal B_{\mathrm{rec}}}
\mathcal E_{\mathrm{low}}^b(\rho_s^b).
\end{align}
The recovery-policy upper bound is
\begin{equation}
\label{eq:recovery-policy-upper-bound}
\begin{aligned}
\bar{\mathcal U}_s^{\mathrm{Rec}}(\mathcal A_s)
&=
\mathcal U_{0,s}
+
\sum_{b\in\mathcal B_{\mathrm{rec}}}
\Big[
\mathbf 1\{b\in\mathcal A_s\}
\mathcal E_{\mathrm{rec}}^b(\rho_s^b)+
\mathbf 1\{b\notin\mathcal A_s\}
\mathcal E_{\mathrm{low}}^b(\rho_s^b)
\Big].
\end{aligned}
\end{equation}
Subtracting \eqref{eq:recovery-policy-upper-bound} from
\eqref{eq:always-low-upper-bound}, the shared term \(\mathcal U_{0,s}\)
cancels.  For any \(b\notin\mathcal A_s\), both bounds contain
\(\mathcal E_{\mathrm{low}}^b(\rho_s^b)\), so its contribution to the difference
is zero.  For any \(b\in\mathcal A_s\), the difference is
\begin{align}
\label{eq:operator-gap-algebra}
\mathcal E_{\mathrm{low}}^b(\rho_s^b)
-
\mathcal E_{\mathrm{rec}}^b(\rho_s^b)
&=
\Delta\mathcal E^b(\rho_s^b),
\end{align}
where the last equality is \eqref{eq:gap-definition}.  Therefore
\eqref{eq:bound_gain_identity} follows.
\end{proof}

\begin{lemma}[Selected-risk persistence lower bound]
\label{lem:selected_persistence_bound}
Under Assumption~\ref{ass:persist}, and assuming
\(\mathcal S_t\subseteq\mathcal A_{t+1}\), for each
\(b\in\mathcal B_{\mathrm{rec}}\),
\begin{equation}
\label{eq:selected-persistence-bound}
\begin{aligned}
&\mathbb E\!\left[
\Delta\mathcal E^b(\rho_{t+1}^b)
\mathbf 1\{b\in\mathcal A_{t+1}\}
\right]\ge
\Pr(b\in\mathcal S_t)\pi_b\Delta\mathcal E^b(\tau_b).
\end{aligned}
\end{equation}
\end{lemma}

\begin{proof}
Since \(\mathcal S_t\subseteq\mathcal A_{t+1}\), we have
\begin{align}
\label{eq:selected-indicator-domination}
\mathbf 1\{b\in\mathcal A_{t+1}\}
&\ge
\mathbf 1\{b\in\mathcal S_t\}.
\end{align}
By \eqref{eq:gap-lower-envelope},
\(\Delta\mathcal E^b(\rho_{t+1}^b)\ge0\), so
\eqref{eq:selected-indicator-domination} implies
\begin{equation}
\label{eq:selected-indicator-expectation-step}
\begin{aligned}
&\mathbb E\!\left[
\Delta\mathcal E^b(\rho_{t+1}^b)
\mathbf 1\{b\in\mathcal A_{t+1}\}
\right]\ge
\mathbb E\!\left[
\Delta\mathcal E^b(\rho_{t+1}^b)
\mathbf 1\{b\in\mathcal S_t\}
\right].
\end{aligned}
\end{equation}
Using \eqref{eq:gap-lower-envelope} again,
\begin{equation}
\label{eq:lower-envelope-expectation-step}
\begin{aligned}
&\mathbb E\!\left[
\Delta\mathcal E^b(\rho_{t+1}^b)
\mathbf 1\{b\in\mathcal S_t\}
\right]\ge
\Delta\mathcal E^b(\tau_b)
\Pr(\rho_{t+1}^b\ge\tau_b,\ b\in\mathcal S_t).
\end{aligned}
\end{equation}
If \(\Pr(b\in\mathcal S_t)=0\), \eqref{eq:selected-persistence-bound} is
immediate.  Otherwise, the product rule and \eqref{eq:selected-risk-persistence}
give
\begin{equation}
\label{eq:persistence-product-step}
\begin{aligned}
&\Pr(\rho_{t+1}^b\ge\tau_b,\ b\in\mathcal S_t)=
\Pr(b\in\mathcal S_t)
\Pr(\rho_{t+1}^b\ge\tau_b\mid b\in\mathcal S_t)\ge
\Pr(b\in\mathcal S_t)\pi_b .
\end{aligned}
\end{equation}
Combining \eqref{eq:selected-indicator-expectation-step},
\eqref{eq:lower-envelope-expectation-step}, and
\eqref{eq:persistence-product-step} gives \eqref{eq:selected-persistence-bound}.
\end{proof}

\begin{proof}[Proof of Theorem~\ref{thm:op-descent-gain-main}]
The identity in \eqref{eq:bound_gain_identity} follows directly from
Lemma~\ref{lem:active_set_gain}.  This is an equality because it is the
algebraic difference between two one-step smoothness upper bounds.

 For the second statement, apply Lemma~\ref{lem:selected_persistence_bound} to
each \(b\in\mathcal B_{\mathrm{rec}}\) and sum over \(b\).  Using
\eqref{eq:bound_gain_identity} at \(s=t+1\), we obtain
\begin{equation}
\label{eq:budgeted_persistence_gain}
\begin{aligned}
&\mathbb E\!\left[
\bar{\mathcal U}_{t+1}^{\mathrm{Low}}
-
\bar{\mathcal U}_{t+1}^{\mathrm{Rec}}(\mathcal A_{t+1})
\right]=
\sum_{b\in\mathcal B_{\mathrm{rec}}}
\mathbb E\!\left[
\Delta\mathcal E^b(\rho_{t+1}^b)
\mathbf 1\{b\in\mathcal A_{t+1}\}
\right]\ge
\sum_{b\in\mathcal B_{\mathrm{rec}}}
\Pr(b\in\mathcal S_t)\pi_b\Delta\mathcal E^b(\tau_b).
\end{aligned}
\end{equation}
This proves \eqref{eq:budgeted_persistence_gain}.
\end{proof}

\begin{lemma}[Variance-gap model as a special case]
\label{lem:variance_gap_special_case}
Suppose the path-induced perturbations are zero-mean, conditionally independent,
and only the second-moment term is retained in the smoothness upper bound.  If
\(\Gamma_b=1\), then
\begin{align}
\label{eq:variance-gap-envelope-special-case}
\mathcal E_q^b(\rho)
&=
\frac{L\eta^2}{2}v_q^b(\rho),
\end{align}
and hence
\begin{equation}
\label{eq:variance-gap-special-case}
\begin{aligned}
\Delta\mathcal E^b(\rho)
&=
\frac{L\eta^2}{2}
\left(
 v_{\mathrm{low}}^b(\rho)-
 v_{\mathrm{rec}}^b(\rho)
\right).
\end{aligned}
\end{equation}
\end{lemma}

\begin{proof}
Under the stated zero-mean assumption, the bias strength satisfies $\mu_q^b(\rho)=0$.
If only the second-moment contribution is retained and \(\Gamma_b=1\), the
perturbation envelope in Remark~\ref{rem:penalty_parameterization} reduces to
\eqref{eq:variance-gap-envelope-special-case}.  Subtracting the recovery-path
penalty from the low-cost-path penalty and using \eqref{eq:gap-definition}
gives \eqref{eq:variance-gap-special-case}.  Thus, the previous variance-gap
analysis is recovered as a special case of the perturbation-envelope model.
\end{proof}

\subsection{Proof of Corollary~\ref{cor:oracle_budget}}

\begin{proof}
By Theorem~\ref{thm:op-descent-gain-main}, the bound-level gain for active set
\(\mathcal A_s\) is $\sum_{b\in\mathcal A_s}\Delta_s^b$.
Therefore, under budget \(|\mathcal A_s|\le K\), the oracle set solves
\begin{equation}
\label{eq:oracle-budget-optimization}
\begin{aligned}
\max_{\mathcal A\subseteq\mathcal B_{\mathrm{rec}}}
\quad
\sum_{b\in\mathcal A}\Delta_s^b
\qquad
\text{s.t.}
\quad
|\mathcal A|\le K .
\end{aligned}
\end{equation}
The objective in \eqref{eq:oracle-budget-optimization} is modular.  If an
optimal set contains \(i\) but excludes \(j\) with \(\Delta_s^j>\Delta_s^i\),
replacing \(i\) with \(j\) strictly increases the summation $\sum_{b\in\mathcal A_s}\Delta_s^b$,
contradicting optimality.  Thus, the oracle selects the largest positive gaps
under the budget.  When all selected gaps are nonnegative and the budget is
filled, this is the top-\(K\) set by \(\Delta_s^b\).
\end{proof}

\subsection{Proof of Theorem~\ref{thm:GNMR-fixed}}
\label{sec:the_part3}

We prove the historical-mean concentration bound for a fixed monitored unit
\(b\).  Let
\begin{align}
\label{eq:gnmr-x-mu-definitions}
X_t^b:=n_{b,t}=\|g_{b,t}\|_2,
\qquad
\mu_b:=\mathbb E[X_t^b]>0 .
\end{align}
Define the normalized variable
\begin{align}
\label{eq:gnmr-normalized-variable}
Y_t^b:=\frac{X_t^b}{\mu_b},
\qquad
\mathbb E[Y_t^b]=1 .
\end{align}
Let
\begin{align}
\label{eq:gnmr-historical-mean}
S_{t-1}^b
&=
\frac{1}{t-1}\sum_{\tau=1}^{t-1}Y_\tau^b
=
\frac{\bar n_{b,t-1}}{\mu_b}.
\end{align}
Then
\begin{align}
\label{eq:gnmr-ratio-representation}
\GNMR_{b,t}
&=
\frac{X_t^b}{\bar n_{b,t-1}}
=
\frac{Y_t^b}{S_{t-1}^b}.
\end{align}

\begin{assumption}[A sufficient sub-exponential model for historical-mean concentration]
\label{ass:psi1-operator}
For the monitored unit \(b\), the centered normalized variables
\(Z_t^b:=Y_t^b-1\) are mean-zero and independent across time over the historical
window considered.  They also satisfy
\begin{align}
\label{eq:psi1-bound}
\|Y_t^b-1\|_{\psi_1}&\le \kappa_b
\end{align}
for some \(\kappa_b>0\) independent of \(t\).  Consequently, there exists an
absolute constant \(c>0\) such that, for every \(x>0\),
\begin{equation}
\label{eq:subexp-tail-bound}
\begin{aligned}
\Pr(Y_t^b-1\ge x)
&\le
\exp\!\Bigg(
-c\min\left\{\frac{x^2}{\kappa_b^2},\frac{x}{\kappa_b}\right\}\Bigg).
\end{aligned}
\end{equation}
\end{assumption}

\begin{remark}[Scope of the sub-exponential assumption]
Assumption~\ref{ass:psi1-operator} is used only as a sufficient local condition for the Bernstein concentration step in Lemma~\ref{lem:mean-conc-operator}. It should not be interpreted as a global claim that operator-level gradient norms are independent or light-tailed throughout training. Sub-exponential tails provide a standard route to Bernstein-type bounds \citep{vershynin2018high}; however, deep-network gradients can be heavy-tailed or nonstationary in some regimes \citep{simsekli2019tail}. In such cases, the same ratio-decomposition argument can be retained by replacing the Bernstein concentration step with a clipped, robust, martingale, mixing, or empirical concentration bound.
\end{remark}

\begin{lemma}[Concentration of the historical mean]
\label{lem:mean-conc-operator}
Under Assumption~\ref{ass:psi1-operator}, for any \(\nu\in(0,1)\) and
\(t\ge2\),
\begin{equation}
\label{eq:historical-mean-concentration}
\begin{aligned}
&\Pr(S_{t-1}^b\le 1-\nu)\le
2\exp\Bigg(
-c(t-1)\min\left\{\frac{\nu^2}{\kappa_b^2},\frac{\nu}{\kappa_b}\right\}\Bigg).
\end{aligned}
\end{equation}
\end{lemma}

\begin{proof}
Let \(Z_\tau=Y_\tau^b-1\).  By Assumption~\ref{ass:psi1-operator},
\(\mathbb E[Z_\tau]=0\), \(\|Z_\tau\|_{\psi_1}\le\kappa_b\), and the sequence is
independent across \(\tau\).  By \eqref{eq:gnmr-historical-mean},
\begin{align}
\label{eq:historical-mean-centered-sum}
S_{t-1}^b-1
&=
\frac{1}{t-1}\sum_{\tau=1}^{t-1}Z_\tau .
\end{align}
The Bernstein inequality for sums of independent sub-exponential random
variables gives
\begin{equation}
\label{eq:bernstein-two-sided-mean}
\begin{aligned}
\Pr(|S_{t-1}^b-1|\ge \nu)\le2\exp\Bigg(
-c(t-1)\min\left\{\frac{\nu^2}{\kappa_b^2},\frac{\nu}{\kappa_b}\right\}\Bigg).
\end{aligned}
\end{equation}
The desired lower-tail bound in \eqref{eq:historical-mean-concentration}
follows from \eqref{eq:bernstein-two-sided-mean} by dropping the upper-tail
event.
\end{proof}

\begin{lemma}[Ratio event decomposition]
\label{lem:ratio_event_decomposition}
Let \(\alpha=1+\varepsilon>1\), and choose \(s_1=\varepsilon/2\) and
\(s_2=\varepsilon/(2\alpha)\).  Then
\begin{equation}
\label{eq:gnmr-ratio-event}
\begin{aligned}
\{\GNMR_{b,t}\ge\alpha\}
\subseteq
\{Y_t^b\ge 1+s_1\}\cup
\{S_{t-1}^b\le 1-s_2\}.
\end{aligned}
\end{equation}
\end{lemma}

\begin{proof}
We prove the contrapositive.  Assume
\begin{align}
\label{eq:ratio-contrapositive-assumption}
Y_t^b<1+s_1,
\qquad
S_{t-1}^b>1-s_2 .
\end{align}
By \eqref{eq:gnmr-ratio-representation} and
\eqref{eq:ratio-contrapositive-assumption},
\begin{align}
\label{eq:ratio-contrapositive-bound}
\GNMR_{b,t}
&=
\frac{Y_t^b}{S_{t-1}^b}
<
\frac{1+s_1}{1-s_2}.
\end{align}
By the choice \(s_1=\varepsilon/2\) and \(s_2=\varepsilon/(2\alpha)\), we have
\begin{equation}
\label{eq:ratio-choice-identity}
\begin{aligned}
\alpha(1-s_2)
&=
\alpha-\frac{\varepsilon}{2}=
1+\frac{\varepsilon}{2}
=
1+s_1 .
\end{aligned}
\end{equation}
Thus \eqref{eq:ratio-choice-identity} implies $\frac{1+s_1}{1-s_2}=\alpha$, which gives \(\GNMR_{b,t}<\alpha\). So the event
\(\{\GNMR_{b,t}\ge\alpha\}\) cannot occur.  This proves
\eqref{eq:gnmr-ratio-event}.
\end{proof}

\begin{proof}[Proof of Theorem~\ref{thm:GNMR-fixed}]
Let \(\alpha=1+\varepsilon>1\), \(s_1=\varepsilon/2\), and
\(s_2=\varepsilon/(2\alpha)\).  By Lemma~\ref{lem:ratio_event_decomposition} and
the union bound,
\begin{equation}
\label{eq:gnmr-union-bound}
\begin{aligned}
\Pr(\GNMR_{b,t}\ge\alpha)
&\le
\Pr(Y_t^b\ge 1+s_1)+
\Pr(S_{t-1}^b\le 1-s_2).
\end{aligned}
\end{equation}
For the first term, \eqref{eq:subexp-tail-bound} gives
\begin{equation}
\label{eq:gnmr-numerator-tail-s1}
\begin{aligned}
\Pr(Y_t^b\ge 1+s_1)
&=
\Pr(Y_t^b-1\ge s_1)\le
\exp\!\left[
-c\min\left\{
\frac{s_1^2}{\kappa_b^2},
\frac{s_1}{\kappa_b}
\right\}
\right].
\end{aligned}
\end{equation}
Since \(s_1=\varepsilon/2\), constants can be absorbed into \(c\), yielding
\begin{equation}
\label{eq:gnmr-numerator-tail-epsilon}
\begin{aligned}
\Pr(Y_t^b\ge 1+s_1)\le
\exp\!\Bigg(
-c\min\left\{
\frac{\varepsilon^2}{\kappa_b^2},
\frac{\varepsilon}{\kappa_b}
\right\}
\Bigg).
\end{aligned}
\end{equation}

For the second term, Lemma~\ref{lem:mean-conc-operator} gives
\begin{equation}
\label{eq:gnmr-denominator-tail-s2}
\begin{aligned}
&\Pr(S_{t-1}^b\le 1-s_2)\le
2\exp\!\Bigg(
-c(t-1)\min\left\{
\frac{s_2^2}{\kappa_b^2},
\frac{s_2}{\kappa_b}
\right\}
\Bigg).
\end{aligned}
\end{equation}

Since \(s_2=\varepsilon/(2\alpha)\), constants independent of
\(\varepsilon,\alpha,\kappa_b,t\) can be absorbed into \(c\), but the factor
\(\varepsilon/\alpha\) must remain. Therefore,
\begin{equation}
\label{eq:gnmr-denominator-tail-epsilon-alpha}
\begin{aligned}
&\Pr(S_{t-1}^b\le 1-s_2)\le2\exp\!\Bigg(
-c(t-1)\min\left\{
\frac{(\varepsilon/\alpha)^2}{\kappa_b^2},
\frac{\varepsilon/\alpha}{\kappa_b}
\right\}
\Bigg).
\end{aligned}
\end{equation}

Combining \eqref{eq:gnmr-union-bound},
\eqref{eq:gnmr-numerator-tail-epsilon}, and
\eqref{eq:gnmr-denominator-tail-epsilon-alpha} proves
\begin{equation}
\label{eq:gnmr_fixed_main}
\begin{aligned}
&\Pr(\GNMR_{b,t}\ge\alpha)
\le
\exp\!\Bigg(
-c\min\left\{
\frac{\varepsilon^2}{\kappa_b^2},
\frac{\varepsilon}{\kappa_b}
\right\}
\Bigg)+2\exp\!\Bigg(
-c(t-1)\min\left\{
\frac{(\varepsilon/\alpha)^2}{\kappa_b^2},
\frac{\varepsilon/\alpha}{\kappa_b}
\right\}
\Bigg).
\end{aligned}
\end{equation}
\end{proof}

\section{Peak-memory modeling under the ${max}O$ budget}
\label{app:memory}

This section provides a simple upper-bound model for the additional peak \emph{activation} memory introduced when the controller promotes a subset of operators to higher precision under a hard $\mathrm{maxO}$ budget.

\noindent \textbf{Setup. }
Let $\mathcal{B}$ be the set of eligible operators.
For each operator $b\in\mathcal{B}$, let $A_b$ denote the number of activation elements that must be saved for backward under a fixed batch and sequence configuration.
At step $t$, let $\mathcal{S}_t\subseteq\mathcal{B}$ be the set of operators executed in high precision.
By design, the budget enforces $|\mathcal{S}_t|\le maxO$ for all $t$.

\medskip\noindent \textbf{From bitwidth to bytes. }
Let $p_{\text{low}}$ and $p_{\text{high}}$ denote the storage precisions in bits per element.
Memory is typically measured in bytes, and one byte equals eight bits.
Therefore, the bytes-per-element conversion is
\[
\text{bytes per element}=\frac{\text{bits per element}}{8}.
\]

\medskip\noindent \textbf{Additional peak activation memory. }
Relative to a baseline that runs all eligible operators in low precision, the additional activation memory at step $t$ is
\begin{equation}
\Delta M_t
=
\frac{p_{\text{high}}-p_{\text{low}}}{8}\sum_{b\in\mathcal{S}_t} A_b,
\label{eq:mem-delta}
\end{equation}
measured in bytes.
Using $|\mathcal{S}_t|\le maxO$, we obtain the upper bound
\begin{equation}
\begin{aligned}
\Delta M_t
\le&
\frac{p_{\text{high}}-p_{\text{low}}}{8}\sum_{b\in \mathrm{TopK}(\mathcal{B},maxO;A_b)} A_b\le
\frac{p_{\text{high}}-p_{\text{low}}}{8}\cdot maxO\cdot A_{\max},
\end{aligned}
\label{eq:mem-upper}
\end{equation}
where $A_{\max}:=\max_{b\in\mathcal{B}}A_b$.

\medskip\noindent \textbf{Interpretation.}
Eq.~\eqref{eq:mem-upper} shows that the additional peak activation memory grows linearly with the budget $\mathrm{maxO}$ and with the activation footprint of the promoted operators.
This explains why a hard $\mathrm{maxO}$ constraint is an effective and model-agnostic knob for controlling peak memory when dynamic precision promotion is triggered by rare but bursty instability events.

\section{Experimental setup and Additional Results}
\label{section: Experimental set up}
This appendix provides training configurations and additional controller-level analyses for the experiments in Section~\ref{section: experiments}. 

\noindent\textbf{Artifact use.}
We use publicly available research datasets and model checkpoints under their respective licenses or access terms. We do not redistribute the original datasets or model checkpoints; the experiments use them for research evaluation and training-stability analysis.
\subsection{Pre-training LLaMA-2 with activation quantization}
\label{section: Experimental set up_aq}
During pre-training across all LLaMA model scales, we implement the configuration framework from \cite{zhao2024galore}, with key technical specifications comprising a 256-token maximum sequence length and a global batch size of 512 samples, corresponding to approximately 131K tokens per optimization step. The learning rate scheduling integrates two-phase optimization: initial linear warm-up during the first 10\% of training iterations, succeeded by cosine decay gradually reducing the learning rate to 10\% of its initial magnitude. Complete architectural configurations and training protocol details have been systematically documented in Table \ref{tab:model-config_aq}.

\begin{table*}[t!]
\centering
\setlength{\tabcolsep}{5pt}
\caption{Model configurations for different LLaMA scales used in GNMR-guided activation quantization.}
\label{tab:model-config_aq}
\begin{tabular}{ccccccc}
\toprule
\textbf{Params} & \textbf{Hidden} & \textbf{Intermediate} & \textbf{Heads} & \textbf{Layers} & 
\begin{tabular}[c]{@{}c@{}}\textbf{Training}\\\textbf{Tokens}\end{tabular} & 
\begin{tabular}[c]{@{}c@{}}\textbf{Learning}\\\textbf{Rate}\end{tabular} \\
\midrule
60M  & 512  & 1376  & 8  & 8  & 1.3B  & 2.5E-3 \\
130M & 768  & 2048  & 12 & 12 & 2.6B  & 2.5E-3 \\
350M & 1024 & 2736  & 16 & 24 & 7.8B  & 1E-3   \\
1.3B   & 2048 & 5461  & 24 & 32 & 13.1B & 6E-4 \\
\bottomrule
\end{tabular}
\end{table*}

\begin{table*}[t!]
\centering
\setlength{\tabcolsep}{5pt}
\caption{Model configurations and final validation perplexity for different LLaMA scales used in GNMR-guided DeepSeek-style training.}
\label{tab:model-config_ds}
\begin{tabular}{ccccccc|c}
\toprule
\textbf{Params} & \textbf{Hidden} & \textbf{Intermediate} & \textbf{Heads} & \textbf{Layers} & 
\begin{tabular}[c]{@{}c@{}}\textbf{Training}\\\textbf{Tokens}\end{tabular} & 
\begin{tabular}[c]{@{}c@{}}\textbf{Learning}\\\textbf{Rate}\end{tabular}  & 
\begin{tabular}[c]{@{}c@{}}\textbf{Validation}\\\textbf{Perplexity}\end{tabular} \\
\midrule
130M & 768  & 2048  & 12 & 12 & 2.6B  & 1E-3 & 24.62 \\
350M & 1024 & 2736  & 16 & 24 & 7.8B  & 5E-4   & 18.87   \\
1.3B   & 2048 & 5472  & 24 & 32 & 13.1B & 5E-4 & 15.20 \\
3B   & 2560 & 6848  & 32 & 32 & 28.8B & 3E-4 & 13.93 \\
\bottomrule
\end{tabular}
\end{table*}

\subsection{Pre-training LLaMA-2 with DeepSeek-style Recipe-Level Recovery}
\label{section: Experimental set up_ds}

This setting evaluates GNMR as a controller layered on top of an existing DeepSeek-style precision recipe. We keep the same sequence length, batch size, and learning-rate scheduler as in Appendix~\ref{section: Experimental set up_aq}. Operators fixed to 16-bit or 32-bit precision by the recipe remain fixed; GNMR only controls the low-precision-eligible subgraph. The detailed model configurations are provided in Table~\ref{tab:model-config_ds}. We use a two-stage GNMR threshold, with a higher threshold during the first 2.5\% of training steps and a lower threshold afterward.
We instantiate two low-cost/recovery hierarchies:

\begin{itemize}
\item BF16 as the higher-fidelity recovery path with Transformer Engine hybrid FP8~\cite{perez2023traininginferencelargelanguage} as the low-cost path;
\item Hybrid FP8 as the higher-fidelity recovery path with quantization-simulated 4-bit (E2M1) and 6-bit (E3M2) linear operators as the low-cost path.
\end{itemize}

This design keeps the backend recipe explicit while testing whether GNMR can make bounded recovery decisions over the low-precision-eligible blocks.

\medskip\noindent\textbf{BF16 High-Precision with Hybrid FP8 Low-Precision Operators. }
In the first configuration, bfloat16 (BF16) serves as the high-precision format, establishing the baseline for parameter storage and computationally intensive operations. For low-precision computations, we implement a hybrid FP8 strategy to balance numerical stability with computational efficiency.

Specifically, the FP8-E4M3 format (4 exponent bits, 3 mantissa bits) is employed during forward propagation, while FP8-E5M2 (5 exponent bits, 2 mantissa bits) is utilized for backward propagation. This mixed-precision framework is implemented through Transformer Engine, which manages both GEMM operations and activation tensor storage in FP8 formats.

Within the FP8 autocast context, BF16 weights and inputs undergo dynamic scaling and conversion to FP8-E4M3 for forward computations. During backward propagation, incoming gradients are cast to FP8-E5M2 to prevent numerical underflow or overflow. All gradients resulting from FP8 GEMM operations are subsequently dequantized back to BF16 for parameter updates.

\medskip\noindent\textbf{Quantization-Simulated Low-Precision Operators. }
In the second configuration, hybrid FP8 operators constitute the high-precision baseline, while low-precision behavior is emulated through quantization using E2M1 (4-bit) and E3M2 (6-bit) formats. For each token, activation values are scaled into the representable range of the target precision, rounded to the nearest representable value, and rescaled to their original dynamic range.

During forward propagation, both weights ($w$) and activations ($x$) undergo quantization prior to computation to simulate precision constraints, with the resulting activations ($y$) being quantized again. Backward propagation follows a similar quantization procedure: incoming gradients ($\partial L/\partial y$) are quantized, and the computed weight and input gradients ($\partial L/\partial w$ and $\partial L/\partial x$) are quantized before further propagation.

\medskip\noindent\textbf{Recoverable scope. }
GNMR controls only the recoverable units exposed by the experimental backend. In the activation-quantization stress bench, the recoverable units are the saved activation paths of attention and SwiGLU MLP projection operators. Components that are fixed by the recipe or kept outside the activation-quantization stress path, such as embeddings, normalization layers, and output heads, are not controller actions. This scope separation keeps the experiment focused on runtime recovery decisions rather than full-model precision allocation.

\begin{table*}[t!]
\caption{The validation perplexity ($\downarrow$) for the pre-training LLaMA-2 models for variance model size with activation quantization under different precision strategy. $\alpha_t$ and $\beta_t$ represent the threshold of GNMR and $\Delta$-GNMR, respectively.}
\centering
\setlength{\tabcolsep}{10pt}
\renewcommand{\arraystretch}{1.1}
\label{table:llama2_activation_quant1}
\begin{threeparttable}
{\small
\begin{tabular}{ccccccc} 
\toprule
\multirow{2}{*}{\textbf{method}} & \multirow{2}{*}{$\alpha_t$} & \multirow{2}{*}{$\beta_t$} & 130M & 350M \\
\cline{4-5}
&&& 2.2B tokens& 6.4B tokens \\ \midrule
\textbf{ADAPT} & N.A. & N.A. & 215.10 & 88.97 \\\midrule
\multirow{3}{*}{$\GNMR$ + $\dGNMR$} & \cellcolor{cyan!30}1.5 & \cellcolor{cyan!30}1.3  & \cellcolor{cyan!30}\textbf{24.66} & \cellcolor{cyan!30}18.84  \\
 & \cellcolor{orange!30}2.0&\cellcolor{orange!30}1.5  & \cellcolor{orange!30}24.86 & \cellcolor{orange!30}19.07 \\
 & \cellcolor{green!30}3.0&\cellcolor{green!30}2.0  & \cellcolor{green!30}25.33 & \cellcolor{green!30}18.77   \\
\bottomrule
\end{tabular}}
\end{threeparttable}
\end{table*}

\subsection{LLaMA-2 13B Fine-tuning Stress Details}
\label{section: Experimental set up_13b}

We evaluate LLaMA-2 13B under three precision settings: BF16, fixed INT8, and dynamic INT8/BF16 controlled by GNMR+\(\Delta\)-GNMR. 
All runs fine-tune LoRA adapters on a frozen LLaMA-2 13B base model using two NVIDIA A100 80G GPUs. 
The common optimization setup uses AdamW, learning rate \(2\times10^{-5}\), weight decay 0.0, gradient clipping 1.0, LoRA rank 8, LoRA scaling \(\alpha_{\mathrm{LoRA}}=16\), and LoRA dropout 0.05. For GNMR+\(\Delta\)-GNMR, the controller acts on the low-precision-eligible recovery path with \(\mathrm{maxO}=10\), \(T_{\mathrm{lock}}=1\), a two-stage GNMR threshold \(\alpha_t=1.5\rightarrow1.1\), and a raw \(\Delta\)-GNMR threshold \(\beta_t=1.5\) with window size 5. 
We report matched downstream evaluations on GSM8K, MMLU, HellaSwag, and WikiText-2: GSM8K uses generation with answer extraction, MMLU uses multiple-choice conditional log-likelihood, HellaSwag uses continuation log-likelihood, and WikiText-2 uses autoregressive language-modeling perplexity. The main downstream results are reported in Table~\ref{tab:13b_downstream}.

\subsection{Comparison of Static Baselines}
\label{app:baseline-comparison-130m}

To compare online control with static sensitivity profiling, we run additional activation-quantization stress experiments on LLaMA-2 130M and 350M with simulated low-bit activation quantization under a standard BF16 training stack. All runs use the same data, optimizer, learning-rate schedule, and token budget as in our main activation-quantization experiments presented in Appendix~\ref{section: Experimental set up_aq}; only the recovery-selection rule changes. Activations of all linear operators in Transformer blocks are quantized using a simulated 4/8-bit floating-point format on top of a BF16 implementation, while weights and optimizer states remain in BF16.

\medskip\noindent\textbf{Static and stage-wise baselines.}
We implement a proxy ADAPT-style~\cite{menon2018adapt} baseline while sharing the same infrastructure as GNMR. Specifically, after a short warm-up phase (500 steps) in BF16, we enable a GNMR/GSS monitor that records per-block gradient-norm-to-history ratios. At the end of warm-up, we aggregate these statistics into a single scalar per block by averaging GNMR over the warm-up steps, sort blocks by this score, and select the top-$25\%$ ``high-risk'' blocks. For the remainder of training, these blocks are always executed with 8-bit activations, while all remaining blocks use 4-bit. The set of high-precision blocks is fixed and no further updates are made. This mimics ADAPT-style approaches that first estimate layer sensitivity and then assign a static bit-width to each layer.

\medskip\noindent\textbf{Experimental results.}
Table~\ref{table:llama2_activation_quant1} presents the validation perplexity during training across different precision strategies. The results show that a fixed sensitivity profile is insufficient under nonstationary runtime risk, while GNMR+\(\Delta\)-GNMR updates the recovery set online and preserves the low-cost trajectory more effectively.

\begin{table*}[t!]
\centering
\setlength{\tabcolsep}{5pt}
\caption{Matched-trigger and overhead comparison on the 130M activation-quantization stress bench.}
\label{table:matched_trigger_overhead_130m}
\begin{tabular}{lccccc}
\toprule
\textbf{Trigger} &\textbf{Final PPL} &
\textbf{Promotion Ratio} & \textbf{Hit-\(\mathrm{maxO}\)} &
\textbf{Step Time (s)} & \textbf{Peak Mem. (MB)} \\
\midrule
GNMR & \textbf{23.685} & \textbf{0.0040} & \textbf{0.0043} & \textbf{0.9972} & 3732 \\
GSNR & 23.729 & -- & -- & 1.2775 & 4009 \\
Jacobian proxy & 23.750 & -- & -- & 1.0167 & 3682 \\
\bottomrule
\end{tabular}
\end{table*}

\begin{table*}[t!]
\centering
\setlength{\tabcolsep}{2.5pt}
\caption{Model configuration for the GPT-based MoE stress test.
GNMR controls only low-precision-eligible MoE components exposed by the backend setting.}
\label{tab:model-config_gpt}
\begin{tabular}{cccccccc}
\toprule
\textbf{Params} & \textbf{Attn Hidden} & \textbf{Per Expert Hidden} & \textbf{MoE Heads} & \textbf{Attn Heads} & \textbf{Topk} & \textbf{Layers} & 
\begin{tabular}[c]{@{}c@{}}\textbf{Learning}\\\textbf{Rate}\end{tabular} \\
\midrule
3.7B  & 7168 & 2048 & 64  & 128  & 8 & 3 & 5E-5 \\

\bottomrule
\end{tabular}
\end{table*}

\begin{table*}[t!]
\caption{Training loss ($\downarrow$) for the GPT-based MoE stress test. GNMR-controlled 8-bit/16-bit recovery is compared with fixed 8-bit and BF16 references. \xmarkk means the run fails to converge.}
\label{table:gpt_training_loss}
\centering
\setlength{\tabcolsep}{4pt}
\renewcommand{\arraystretch}{1.2}
\begin{threeparttable}
\begin{tabular}{c|cccccccccccc} 
\toprule
\multirow{2}{*}{\textbf{Precision}} & \multicolumn{11}{c}{Training steps (K)} \\
\cline{2-13}  
& 1  & 2 & 3 & 4 & 6 & 8 & 10 & 12 & 14 & 17 & 20 & 29 \\ 
\midrule
\textbf{8-bit}              & 0.223 & 0.205 & 0.189 & 0.312 & 0.176 & 0.170 & 0.167 & 0.163 & 0.160 & 0.159 & 0.158 & \xmarkk \\
\cellcolor{orange!30}\textbf{8-bit/16-bit} & \cellcolor{orange!30}0.224 & \cellcolor{orange!30}0.201 & \cellcolor{orange!30}0.189 & \cellcolor{orange!30}0.183 & \cellcolor{orange!30}0.173 & \cellcolor{orange!30}0.170 & \cellcolor{orange!30}0.168 & \cellcolor{orange!30}0.164 & \cellcolor{orange!30}0.162 & \cellcolor{orange!30}0.159 & \cellcolor{orange!30}0.159 & \cellcolor{orange!30}0.154 \\  
\textbf{BF16}& 0.226 &  0.205 & 0.187 & 0.181 & 0.172 & 0.171 & 0.167 & 0.163 & 0.159 & 0.159 & 0.156 & 0.152 \\ 
\bottomrule
\end{tabular}
\end{threeparttable}
\end{table*}

\begin{table*}[t!]
\centering
\setlength{\tabcolsep}{5pt}
\caption{Model configurations for PanGu-1B used in GNMR-guided 4-bit training.}
\label{tab:model-config_pangu}
\begin{tabular}{cccccccc}
\toprule
\textbf{Params} & \textbf{Hidden} & \textbf{Intermediate} & \textbf{Heads} & \textbf{Layers} & 
\textbf{KV groups} & \textbf{Channels} &
\begin{tabular}[c]{@{}c@{}}\textbf{Learning}\\\textbf{Rate}\end{tabular} \\
\midrule
1.2B  & 1536  & 6144  & 12  & 26 & 6 & 128 & 2E-3 \\
\bottomrule
\end{tabular}
\end{table*}

\begin{table*}[t!]
\caption{Training loss ($\downarrow$) for the PanGu-1B-class 4-bit stress test. GNMR-controlled recovery improves the fixed 4-bit low-cost path, while BF16 remains the higher-fidelity reference.}
\label{table:pangu1b_training_loss}
\centering
\setlength{\tabcolsep}{4pt}
\renewcommand{\arraystretch}{1.2}
\begin{threeparttable}
\begin{tabular}{c|ccccccccccc} 
\toprule
\multirow{2}{*}{\textbf{Precision}} & \multicolumn{11}{c}{Training Tokens (B)} \\
\cline{2-12}  
& 0   & 0.4 & 0.8 & 1.2 & 1.6 & 2.0 & 2.4 & 2.8 & 3.2 & 3.6 & 4.0 \\ 
\midrule
\textbf{4-bit}              & 12.261 & 4.656 & 4.177 & 3.670 &3.527  & 3.348 & 3.314 & 3.238 & 3.170 & 3.156 & 3.124 \\
\cellcolor{orange!30}\textbf{4-bit+GNMR}  & \cellcolor{orange!30}12.261 & \cellcolor{orange!30}4.650 &\cellcolor{orange!30}3.929 &\cellcolor{orange!30}3.575 &\cellcolor{orange!30}3.564 &\cellcolor{orange!30}3.309 &\cellcolor{orange!30}3.290 &\cellcolor{orange!30}3.215 &\cellcolor{orange!30}3.155 &\cellcolor{orange!30}3.128 &\cellcolor{orange!30}3.076 \\  
\textbf{BF16}& 12.261   & 4.608 & 3.878 & 3.526 & 3.472 & 3.210 & 3.160 & 3.077 & 3.004 & 2.978 & 2.913 \\ 
\bottomrule
\end{tabular}
\end{threeparttable}
\end{table*}

\subsection{130M Matched-trigger and Overhead Characterization}
\label{app:matched-trigger-overhead-130m}

\noindent \textbf{Setup.}
We further isolate trigger quality from the recovery actuator on the 130M activation-quantization stress bench. 
This experiment follows the same 130M pre-training protocol as Appendix~\ref{section: Experimental set up_aq}: LLaMA-style 130M pre-training on C4-en~\citep{touvron2023llama,raffel2020exploring}, following the GaLore/CR-Net training setup~\citep{zhao2024galore,kong2025cr}, with maximum learning rate \(3\times10^{-3}\) and 2.6B training tokens. 
The low-cost path uses 4-bit activation quantization and the recovery path uses 8-bit activation quantization over the same recoverable projection-activation units as in the main 130M stress setting. 

All compared methods share the same model, data, optimizer, token budget, low-cost/recovery activation paths, and active recovery cap \(\mathrm{maxO}=4\). 
GNMR uses thresholded online triggering with \(T_{\mathrm{lock}}=1\), so \(\mathrm{maxO}=4\) is an upper bound rather than a target to fill. 
For GSNR and the Jacobian proxy, we use budgeted top-\(\mathrm{maxO}\) ranking baselines that select four units per step. 
Thus, this comparison tests whether GNMR can achieve competitive quality with a selective thresholded controller, rather than simply comparing full-budget rankers.

\medskip\noindent \textbf{Compared triggers.}
We compare GNMR with two alternative operator-level signals: a GSNR-style gradient stability signal~\citep{jiang2023accelerating} and a Jacobian-based proxy~\citep{takase2023spike}. 
We do not include GSS in this matched-trigger table because GSS is a global/full-process spike score rather than a per-unit online trigger under the same recovery interface. 
Step time and peak memory characterize overhead under the reported software stack.

\noindent \textbf{Results and interpretation.}
Table~\ref{table:matched_trigger_overhead_130m} shows that GNMR is an effective thresholded online trigger under the shared 4-bit/8-bit activation recovery interface. 
It obtains the best final PPL while activating the recovery path only sparsely: the promotion ratio is \(0.0040\), and the hit-\(\mathrm{maxO}\) rate is \(0.0043\). 

Under the same benchmark, GNMR also has the lowest measured step time and lower peak memory than GSNR. 
The Jacobian proxy has a comparable memory footprint, but gives worse final PPL under a denser full-budget ranking policy. 
Overall, this comparison supports GNMR as a selective and budget-efficient runtime trigger for the tested controller interface, rather than merely a correlated stability score.

\subsection{GPT-based MoE Stress Test}
\label{section: Experimental set up_gpt}
\noindent \textbf{Experiment setup.}
We include a GPT-based MoE stress test to evaluate whether the same runtime controller interface remains useful beyond dense LLaMA-style models. 
We pre-train a 3.7B-parameter GPT-based model~\citep{radford2019language,brown2020language} with a Mixture-of-Experts (MoE) architecture~\citep{shazeer2017outrageously,jacobs1991adaptive}. 
The model is trained with global batch size 32, sequence length 4096, learning rate \(5\times10^{-5}\), and cosine decay. 
In the MoE layers, GNMR controls whether low-precision-eligible components use the low-cost 8-bit path or the 16-bit recovery path. 
We compare this GNMR-controlled 8-bit/16-bit recovery policy with fixed 8-bit and BF16 references. 
We use \(\alpha_{\mathrm{GNMR}}=0.9\) after a high-fidelity warm-up: during the first 5\% of training steps, all low-precision-eligible MoE components use the 16-bit recovery path. 
This warm-up avoids reacting to unusually large early gradients before the GNMR running statistics stabilize; after it, GNMR controls the 8-bit/16-bit recovery decisions under the same threshold.
Complete architectural configurations and training details are provided in Table~\ref{tab:model-config_gpt}.

\medskip\noindent \textbf{Experiment results.}
Table~\ref{table:gpt_training_loss} reports training loss under the fixed 8-bit low-cost path, GNMR-controlled 8-bit/16-bit recovery, and BF16 reference. 
The fixed 8-bit run becomes unstable and fails to converge, while GNMR-controlled recovery remains close to the BF16 reference. 
This is a controller-level architecture stress test: fixed low-precision execution can become unstable in MoE training, while GNMR converts local runtime risk signals into bounded recovery actions. 
The result supports the portability of the runtime recovery interface to a non-dense transformer setting.


\subsection{PanGu-1B-class 4-bit Stress Test}
\label{section: Experimental set up_pangu}
We include a PanGu-1B-class stress test under 4-bit low-precision pressure. 
The model is trained with global batch size 256, sequence length 4096, learning rate \(2\times10^{-3}\), and cosine decay. 
We compare a fixed 4-bit low-cost path, a GNMR-controlled recovery setting, and a BF16 reference. 
The GNMR threshold is fixed to 0.4, and the model is trained for 4B tokens. 
The architectural configuration is provided in Table~\ref{tab:model-config_pangu}.

\medskip\noindent\textbf{Experiment results. }
Table~\ref{table:pangu1b_training_loss} reports training loss under the fixed 4-bit low-cost path, GNMR-controlled recovery, and BF16 reference. 
GNMR-controlled recovery improves convergence relative to fixed 4-bit training, while BF16 remains the higher-fidelity reference. 
This result supports the controller interface beyond the LLaMA-style setting: GNMR applies risk-triggered recovery whenever the backend exposes a low-cost path and a recovery path.

\end{document}